\documentclass[11pt]{article}

\usepackage{graphicx}
\usepackage{amsmath}
\usepackage{amsfonts}
\usepackage{amssymb}
\usepackage{amsthm}
\usepackage{mathtools}
\usepackage{booktabs}
\usepackage{float}
\usepackage[font=small,labelfont=bf]{caption}
\usepackage{subcaption}
\usepackage[space]{grffile}
\usepackage{url}
\usepackage{algorithm}
\usepackage{algorithmic}
\usepackage[round]{natbib}
\usepackage[margin=1in]{geometry}
\usepackage{colortbl}
\usepackage[dvipsnames]{xcolor}
\usepackage[
  colorlinks=true,
  linkcolor=blue,
  citecolor=blue!60!black,
  urlcolor=blue!60!black,
  filecolor=black
]{hyperref}

\newtheorem{theorem}{Theorem}[section]
\newtheorem{lemma}[theorem]{Lemma}
\newtheorem{proposition}[theorem]{Proposition}

\theoremstyle{definition}
\newtheorem{definition}[theorem]{Definition}
\theoremstyle{remark}
\newtheorem{remark}[theorem]{Remark}

\DeclareMathOperator*{\argmax}{arg\,max}

\title{Annealed Softmax Greedy in Many-Armed Bayesian Bandits}

\author{William Overman \\ Stanford University \and Mohsen Bayati \\ Stanford University}
\date{}

\begin{document}
\maketitle
\begingroup
  \renewcommand\thefootnote{}
  \footnotetext{Emails: \texttt{wpo@stanford.edu}, \texttt{bayati@stanford.edu}.}
  \addtocounter{footnote}{-1}
\endgroup

\begin{abstract}
Reinforcement learning with verifiable rewards (RLVR) and group-based policy optimization methods such as GRPO update a stochastic policy by sampling multiple completions per prompt and increasing the policy's probability on those with higher reward, regularized by a KL penalty toward a reference policy. These updates, unline the exploration mechanism in Thompson sampling and UCB, do not include explicit mechanisms that track epistemic uncertainty.
This paper studies a stylized explanation for why such uncertainty-agnostic updates can nevertheless be effective.
We analyze an annealed softmax (Boltzmann) policy that selects actions according to a softmax of empirical mean rewards in a many-armed Bayesian Bernoulli bandit.
Under a linear upper-tail condition on the prior (the $\beta=1$ case of $\beta$-regularity), which implies an abundance of near-optimal arms, we prove that annealed softmax greedy achieves Bayes regret $\tilde{O}(m + T/m)$, and in particular $\tilde{O}(\sqrt{T})$ when the number of arms scales as $m = \Theta(\sqrt{T})$. This is the near-optimal Bayes regret rate in this regime, attained also by empirical-mean greedy.
Under $\beta$-regularity many arms keep empirical means near the optimum throughout learning, so the probability that softmax places away from the empirical best falls mostly on other near-optimal arms.
By contrast, with a small number of arms, the same kind of softmax policy can suffer linear regret \citep{cesabianchi2017boltzmann}. The result also provides a structural analogy to RLVR, where a base policy with a non-negligible probability of producing a correct completion plays the role of $\beta$-regularity.
Simulations support the theory and motivate prior-anchored variants of greedy and annealed softmax that score arms by the Beta posterior mean and skip the forced initialization; with an arm-specific prior, accurate or noisy, these variants outperform baselines, including Thompson Sampling, when the number of arms is large.
\end{abstract}

\section{Introduction}
\label{sec:introduction}

Reinforcement learning with verifiable rewards (RLVR) is now a common component of language-model post-training pipelines. In RLVR, a model is treated as a stochastic policy over candidate solutions and is trained using rewards that can be checked automatically, such as exact-match accuracy in mathematics or unit-test success in code. This makes it possible to optimize behavior at scale without relying on human preference labels. Group-based policy optimization methods such as Group Relative Policy Optimization (GRPO) are prominent examples of this paradigm \citep{shao2024deepseekmath}.

RLVR sits somewhat outside the classical exploration picture in reinforcement learning and bandit theory. In those settings, performance is often tied to \emph{epistemic uncertainty}: the learner must gather information to distinguish promising actions from poor ones. Group-based RLVR pipelines, by contrast, sample multiple completions per prompt, raise the policy probability of those with higher reward, and apply a KL penalty toward a reference policy; the randomness comes from repeated sampling rather than from an uncertainty-aware exploration rule. Under verifiable rewards, such updates can be viewed as iteratively amplifying the probability of already successful outputs, with the gain coming largely from redistribution within the model's existing support \citep{liu2025r1zero}. This raises a basic question: when can uncertainty-agnostic reweighting work well in problems that seem to require exploration?

A related discussion concerns the base model's \emph{coverage} of good solutions. Here pass@$k$ denotes the probability that at least one of $k$ independent samples from the model is correct, with pass@1 being the per-sample success rate. One line of work reports that RLVR often improves pass@1 while making little progress on pass@$k$ for large $k$, suggesting that rewarded reasoning paths may already be present in the base model's output distribution \citep{yue2025rlvr}. Other recent work reports settings where RLVR \emph{does} extend reasoning performance, indicating that the empirical picture depends on training details and evaluation choices \citep{cui2025entropy}. We do not try to resolve that broader debate. We instead study a simplified model that isolates one mechanism by which reweighting alone can be effective.

\subsection{A stylized lens: many-armed bandits and annealed softmax}

We consider a stochastic many-armed bandit model with Bernoulli rewards and Beta priors, where each arm represents a ``completion mode'' and reward corresponds to a verifiable success event. Within this model, we analyze an annealed softmax policy that, at each time step $t$, selects an arm with probability proportional to $\exp(\eta_t m_{i,t})$, where $m_{i,t}$ is the empirical mean reward of arm $i$ at time $t$ and $\eta_t$ is an inverse-temperature schedule that increases over time, a setup referred to as \emph{annealing}. The policy does not use an optimism term, a posterior sample, or per-arm confidence intervals.

A classical result of \citet{cesabianchi2017boltzmann} shows that this type of exploration (also referred to as Boltzmann exploration) can suffer linear regret with a small number of arms under any monotone temperature schedule, and that strong guarantees in that regime typically require schedules that adapt to arm-specific information. The standard failure mode is ``cooling too fast'': early noise gets amplified, causing premature concentration on suboptimal actions. At first glance, such results argue against exactly the kind of uncertainty-agnostic annealing that appears in RLVR practice.

The many-armed Bayesian setting allows a different possibility. If the prior places enough mass near the optimum, the action space contains many arms whose rewards are already close to optimal. Even if the algorithm spreads probability across several arms, many of them are good enough that the regret cost remains small. This idea is studied in the many-armed bandit literature, where greedy-style procedures achieve Bayesian regret guarantees under upper-tail regularity conditions on the prior \citep{bayati2020greedy}.

\subsection{Main results and contributions}

The question we study is therefore: can annealed softmax, despite ignoring epistemic uncertainty, inherit the regret behavior of greedy methods in the many-armed Bayesian regime? We answer this positively. The main contributions are:

\begin{enumerate}
    \item \emph{Bayes regret guarantee.} Under the linear upper-tail condition on the prior (the $\beta=1$ case of $\beta$-regularity), annealed softmax achieves Bayes regret $\tilde{O}(m + T/m)$, yielding $\tilde{O}(\sqrt{T})$ when $m = \Theta(\sqrt{T})$. This is the near-optimal Bayes regret rate in this regime (also attained by empirical-mean greedy \citep{bayati2020greedy}), achieved using only empirical-mean scores, with no optimism term, no posterior sample, and no per-arm confidence interval. 

    \item \emph{Mechanism.} Annealed softmax differs from greedy only by placing some probability mass on arms that are not currently empirically best. Under $\beta$-regularity, many arms maintain empirical means close to the optimum throughout learning, so when softmax samples an arm other than the empirically best, that arm tends to be another near-optimal one rather than a clearly inferior one. This contrasts with the failure mode studied by \citet{cesabianchi2017boltzmann}, where the same kind of policy can suffer linear regret with a small number of arms.

    \item \emph{Structural analogue in RLVR.} The $\beta$-regular prior condition has a direct analogue in RLVR: a base policy with non-negligible probability of producing a correct completion plays the role of $\beta$-regularity, with pass@$k$ probing the upper tail of the policy's completion distribution. In that setting, repeated sampling surfaces correct completions reliably, and subsequent reweighting can improve pass@1 without an explicit uncertainty-aware exploration rule. We note that this is only a structural analogy and not a formal result about GRPO; the theorem lives in the non-contextual bandit model, and extending the analysis to the contextual or sequential settings where GRPO operates is an open problem \citep{shao2024deepseekmath,liu2025r1zero,yue2025rlvr,cui2025entropy}.

    \item \emph{Prior-anchored variants.} Simulations  (Section~\ref{sec:experiments}) confirm the theorem's prediction and expose a practical weakness shared by greedy and annealed softmax greedy (ASG): both must pull every arm once before scoring can begin, which costs on the order of $m$ regret when arms are plentiful. We evaluate variants, Prior-Greedy and Prior-ASG, that score arms by the Beta posterior mean instead of the empirical mean and skip the initialization. With an arm-specific prior, even a noisy one, they outperform every baseline including Thompson Sampling and UCB at large $m$. Appendix~\ref{app:prior-theory} outlines how the main theorem should extend to this setting, with the $m$ term of the regret bound removed.
\end{enumerate}

\subsection{Paper organization}

Section~\ref{sec:related} reviews related work. Section~\ref{sec:setup} introduces the many-armed Bayesian Bernoulli bandit model, the $\beta$-regular prior assumption, and notation. Section~\ref{sec:algorithms} presents the greedy baseline and the annealed softmax greedy (ASG) algorithm. Section~\ref{sec:results} states the main regret theorem and its extensions. Section~\ref{sec:experiments} reports simulations and introduces the prior-anchored variants Prior-Greedy and Prior-ASG. Section~\ref{sec:discussion} concludes with discussion. Proofs and additional simulations, as well as discussions, are presented in the appendices.

\section{Related Work}
\label{sec:related}

This paper connects three strands of literature: the theory of Boltzmann exploration in bandits, the many-armed bandit framework with ``free exploration,'' and the empirical RLVR/GRPO pipeline for LLM post-training.
We give a brief overview here and defer a longer survey to Appendix~\ref{app:related}.

\paragraph{Boltzmann exploration and its limitations.}
Boltzmann (softmax) action selection is a standard randomized alternative to greedy or optimistic policies in bandits and RL.
\citet{cesabianchi2017boltzmann} prove that, for fixed-$K$ stochastic $K$-armed bandits, \emph{any monotone} temperature schedule can leads to suboptimal behavior, either exploring too long or committing too early, and propose per-arm learning rates that explicitly track uncertainty as a remedy.
This negative result is the starting point of our analysis: we identify a structural regime (many arms, thick-tailed prior) that circumvents the impossibility without requiring uncertainty-aware schedules.

\paragraph{Many-armed bandits and free exploration.}
When the number of arms is large relative to the horizon, the distribution of arm qualities (rather than per-arm estimation) governs achievable regret.
In the Bayesian setting, \citet{bayati2020greedy} show that under an upper-tail regularity condition on the prior (``$\beta$-regularity''), a subsampled greedy policy achieves Bayes regret $\tilde{O}(\max\{m,T/m\})$, yielding $\tilde{O}(\sqrt{T})$ with $m=\Theta(\sqrt{T})$ arms.
The key mechanism is \emph{free exploration}: discarding a poorly performing arm still leaves many near-optimal alternatives.
Our work extends this viewpoint from greedy to annealed softmax policies, showing that the same prior-tail structure suppresses the ``softmax leakage'' that causes failures in the fixed-$K$ regime.
Related infinite-armed formulations include \cite{berry1997infinitely,wang2009algorithms,carpentier2015simple}.

\paragraph{RLVR, GRPO, and the ``bounded-by-base'' phenomenon.}
Reinforcement learning with verifiable rewards has become a core post-training technique for reasoning-oriented LLMs.
GRPO \cite{shao2024deepseekmath} performs group-sampled policy updates under KL regularization, and recent analyses relate its dynamics to iterative soft reweighting of completions \cite{mroueh2025grpo,liu2025r1zero}.
A central empirical observation is that RLVR often improves pass@1 while failing to improve, or even degrading, large-$k$ pass@$k$, suggesting that optimization primarily redistributes probability mass within the base model's existing support rather than discovering new reasoning paths \cite{yue2025rlvr}; though subsequent work argues this picture is nuanced and depends on training details and evaluation choices \cite{cui2025entropy,wen2025implicit,liu2025prorl}.
Our bandit model provides a stylized formalization of this ``bounded-by-base'' effect: the $\beta$-regular tail condition translates ``good pass@$k$'' into ``many near-optimal arms,'' under which softmax reweighting suffices without explicit exploration.

\section{Problem Setup}
\label{sec:setup}

This section sets up the model and the notation. More specifically, it formalizes the bandit model, defines the prior regularity condition and the Bayes regret criterion, and introduces the empirical-mean notation used in the algorithm. Throughout the paper we assume $T \ge m$.

\subsection{Many-armed Bayesian Bernoulli bandit}
Fix a time horizon $T\in\mathbb{N}$ and consider $m\in\mathbb{N}$ arms indexed by $i\in[m] := \{1,\dots,m\}$.
Each arm has an unknown mean $\mu_i\in[0,1]$ drawn i.i.d.\ from a prior $\Gamma$ on $[0,1]$.
Conditional on $\mu_i$, each arm $i$ has an i.i.d.\ sequence of Bernoulli rewards
\[
X_{i,s}\mid \mu_i \ \stackrel{\text{i.i.d.}}{\sim}\ \mathrm{Bernoulli}(\mu_i)\qquad (s=1,2,\dots),
\]
where $X_{i,s}$ is the reward from the $s$-th pull of arm $i$.
At each round $t\in[T]$, a policy selects an arm $A_t\in[m]$ (possibly randomized and history-dependent) and observes $X_{A_t, N_{A_t}(t)}$, where
\[
N_i(t) := \sum_{s=1}^t \mathbf{1}\{A_s=i\}
\]
is the number of pulls of arm $i$ up to time $t$.

\subsection{Prior regularity: $\beta$-regular upper tail}
We assume $\Gamma$ has a polynomially thick upper tail near $1$.

\begin{definition}[$\beta$-regular prior near $1$]\label{def:beta-regular}
A distribution $\Gamma$ on $[0,1]$ is \emph{$\beta$-regular} if there exist constants $0<c_0\le C_0<\infty$ and $\varepsilon_0\in(0,1)$ such that, for all $\varepsilon\in(0,\varepsilon_0]$,
\[
c_0\,\varepsilon^\beta \ \le\ \Gamma([1-\varepsilon,1]) \ \le\ C_0\,\varepsilon^\beta.
\]
\end{definition}

Similar to \citet{bayati2020greedy}, we will present the main theorem first for the \emph{linear tail} case $\beta=1$, and then state a general-$\beta$ extension as a xremark.

\subsection{Bayes regret and a convenient surrogate}
Let $\mu_* := \max_{i\in[m]} \mu_i$.
The (frequentist) regret of a policy $\pi$ given $\boldsymbol{\mu}=(\mu_1,\dots,\mu_m)$ is
\[
R_T(\pi\mid \boldsymbol{\mu}) \ :=\ \sum_{t=1}^T \big(\mu_* - \mu_{A_t}\big).
\]
The \emph{Bayes regret} is
\[
\mathrm{BR}_{T,m}(\pi)\ :=\ \mathbb{E}\big[R_T(\pi\mid \boldsymbol{\mu})\big],
\]
where the expectation is over $\boldsymbol{\mu}\sim\Gamma^{\otimes m}$ (i.e., $\mu_1,\dots,\mu_m$ are i.i.d.\ draws from $\Gamma$, so $\Gamma^{\otimes m}$ is the $m$-fold product measure on $[0,1]^m$), reward randomness, and any internal randomness of $\pi$.

We will also use the standard surrogate ``regret-to-$1$''
\begin{equation}\label{eq:regret-to-1}
\widetilde{R}_T(\pi\mid\boldsymbol{\mu}) \ :=\ \sum_{t=1}^T (1-\mu_{A_t}) \ =\ \sum_{i=1}^m (1-\mu_i)\,N_i(T),
\end{equation}
which satisfies the exact identity
\begin{equation}\label{eq:proxy-ineq}
R_T(\pi\mid\boldsymbol{\mu}) \ =\ \widetilde{R}_T(\pi\mid\boldsymbol{\mu}) \ -\ T(1-\mu_*).
\end{equation}
Thus
\begin{equation}\label{eq:proxy-ineq-bayes}
\mathrm{BR}_{T,m}(\pi) \ = \ \mathbb{E}\big[\widetilde{R}_T(\pi\mid\boldsymbol{\mu})\big] \ -\ T\,\mathbb{E}[1-\mu_*]
\ \le\ \mathbb{E}\big[\widetilde{R}_T(\pi\mid\boldsymbol{\mu})\big].
\end{equation}
Under 1-regularity, $\mathbb{E}[1-\mu_*]$ is of order $1/m$ (Lemma~\ref{lem:order-stat} below), so the exact identity can sharpen constants, but our main upper bound only needs $\mathrm{BR}_{T,m}(\pi)\le \mathbb{E}[\widetilde{R}_T(\pi\mid\boldsymbol{\mu})]$.

\subsection{Empirical means}
Let $S_i(t)$ be the number of observed successes of arm $i$ up to time $t$, so $0\le S_i(t)\le N_i(t)$.
Because every arm is pulled once during initialization in Algorithms~\ref{alg:greedy}--\ref{alg:asg}, we have $N_i(t)\ge 1$ throughout the greedy/softmax phase.
We therefore use the empirical mean
\begin{equation}\label{eq:post-mean}
\widehat{\mu}_{i,t} \ :=\ \frac{S_i(t)}{N_i(t)}
\end{equation}
as the score of arm $i$ at time $t$.

\paragraph{Asymptotic notation.} Throughout the paper, $\tilde{O}(\cdot)$ hides polylogarithmic factors in $T$ and $m$.

\section{Algorithms}
\label{sec:algorithms}

This section presents the greedy baseline and the annealed softmax greedy (ASG) policy that is the main object of analysis.

\subsection{Greedy baseline on empirical means}

We begin with the natural greedy benchmark, shown in Algorithm~\ref{alg:greedy}. After pulling each arm once so that every arm has an initial estimate, the policy always selects an arm with the largest empirical mean. This is exactly the many-armed greedy rule analyzed by \citet{bayati2020greedy}, and it serves as the baseline throughout the paper.

\begin{algorithm}[h!]
	\caption{Greedy (Empirical-Mean Greedy)}
	\label{alg:greedy}
	\begin{algorithmic}[1]
		\REQUIRE Arms $[m]$, horizon $T$
		\STATE \textbf{Initialization:}
		\FOR{$t = 1, \dots, m$}
		\STATE Pull arm $A_t = t$; observe reward $X_{t,1}$
		\ENDFOR
		\STATE \textbf{Greedy phase:}
		\FOR{$t = m+1, \dots, T$}
		\STATE Compute empirical means $\widehat{\mu}_{i,t-1} = \frac{S_i(t-1)}{N_i(t-1)}$ for all $i \in [m]$
		\STATE Pull arm $A_t \in \argmax_{i \in [m]}\, \widehat{\mu}_{i,t-1}$ \quad (ties broken arbitrarily)
		\STATE Observe reward $X_{A_t, N_{A_t}(t)}$
		\ENDFOR
	\end{algorithmic}
\end{algorithm}

\subsection{Annealed Softmax Greedy on empirical means}

Our main object of study is a randomized analogue of greedy, given in Algorithm~\ref{alg:asg}. Instead of deterministically choosing the empirical-best arm at each round, the policy samples from a softmax distribution over empirical means. Arms with higher empirical means are more likely to be selected, but lower-ranked arms still receive some probability mass. The amount of randomness is controlled by a nondecreasing inverse-temperature schedule $\{\eta_t\}_{t\ge 1}$.

This policy can be viewed as interpolating between uniform sampling and pure greedy behavior. When $\eta_t$ is small, the softmax distribution is relatively flat, so the algorithm spreads probability more evenly across all arms. As $\eta_t$ grows, the distribution becomes more concentrated, and the policy increasingly resembles Greedy. The key question for the rest of the paper is whether this softmax sampling rule, which is agnostic to epistemic uncertainty in arm quality, can achieve near-greedy regret in the many-armed Bayesian regime.

ASG uses exactly the same empirical information as the greedy baseline; it neither estimates uncertainty explicitly nor uses per-arm temperatures that depend on the pull counts $N_i(t)$. \citet{cesabianchi2017boltzmann} prove that any monotone temperature schedule shared across arms can suffer linear regret in some small-arm instance, and propose per-arm scalings as a remedy: their Boltzmann--Gumbel Exploration uses $\sigma/\sqrt{N_i(t)}$, so less-pulled arms have noisier scores and are favored for exploration. We study the unadjusted version because it matches the structure used in group-based methods such as GRPO, where the policy does not adapt its per-completion sampling temperature to historical counts.

\paragraph{Concrete schedules.}
For the analysis, it is convenient to choose a logarithmically increasing schedule of the form
\begin{equation}\label{eq:schedule}
\eta_t \ =\ \frac{c_\eta}{\delta}\log(t\vee 2),
\end{equation}
where $c_\eta>1$ is a fixed constant and $\delta$ is a threshold parameter that will be tuned as a function of $(T,m)$ in the regret bound. Intuitively, $\delta$ sets the scale of what counts as a meaningfully suboptimal arm in the proof. With this choice,
\(
\exp(-\eta_t \delta)=(t\vee 2)^{-c_\eta},
\)
so the probability weight placed on arms that are worse by at least $\delta$ decays polynomially in time. In particular,
\[
\sum_{t=1}^T \exp(-\eta_t\delta)=O(1),
\]
which is exactly what we need to make the total softmax ``leakage'' summable. In the linear-tail case ($\beta=1$), for example, the relevant choice is $\delta \asymp (\log T)/m$, leading to the $\tilde{O}(m + T/m)$ Bayes regret rate (Theorem~\ref{thm:asg-beta1}).

\begin{algorithm}[t!]
\caption{Annealed Softmax Greedy (ASG)}
\label{alg:asg}
\begin{algorithmic}[1]
\REQUIRE Arms $[m]$, horizon $T$, nonnegative nondecreasing inverse-temperature schedule $\{\eta_t\}_{t \ge 1}$ with $\eta_t \to \infty$
\STATE \textbf{Initialization:}
\FOR{$t = 1, \dots, m$}
    \STATE Pull arm $A_t = t$; observe reward $X_{t,1}$
\ENDFOR
\STATE \textbf{Softmax phase:}
\FOR{$t = m+1, \dots, T$}
    \STATE Compute empirical means $\widehat{\mu}_{i,t-1}$ for all $i \in [m]$
    \STATE Sample arm $A_t = i$ with probability
    \begin{equation}\label{eq:softmax}
    p_{t}(i) \ :=\ \frac{\exp(\eta_t\, \widehat{\mu}_{i,t-1})}{\sum_{j=1}^m \exp(\eta_t\, \widehat{\mu}_{j,t-1})}
    \end{equation}
    \STATE Observe reward $X_{A_t, N_{A_t}(t)}$
\ENDFOR
\end{algorithmic}
\end{algorithm}

\section{Main Results}
\label{sec:results}

This section presents the main Bayes regret bound for ASG. 

Before stating our main theorem, we recall the result for the greedy policy (Algorithm~\ref{alg:greedy}). In the Bernoulli many-armed setting with a linear upper tail, \citet{bayati2020greedy} show that empirical-mean greedy achieves near-optimal Bayes regret.

\begin{proposition}[Greedy benchmark in the Bernoulli, $\beta=1$ regime {\citep[Theorem~\textup{(Bernoulli)}]{bayati2020greedy}}]\label{prop:greedy-benchmark}
Assume Bernoulli rewards and a $1$-regular prior $\Gamma$ (Definition~\ref{def:beta-regular} with $\beta=1$).
Then Greedy (Algorithm~\ref{alg:greedy}) achieves
\[
\mathrm{BR}_{T,m}(\text{Greedy}) \ \le\ \tilde{O}\!\left(m + \frac{T}{m}\right),
\]
and in particular for $m=\Theta(\sqrt{T})$ one has $\mathrm{BR}_{T,m}(\text{Greedy})=\tilde{O}(\sqrt{T})$.
\end{proposition}

\subsection{ASG matches greedy up to a leakage term}

We now turn to our main result for ASG. At a high level, the theorem shows that ASG behaves like greedy plus an additional \emph{softmax leakage} penalty: the probability mass assigned away from the empirically best arms. The key technical point is that in the many-armed, thick-tail regime, this leakage remains controlled because when softmax samples an arm other than the empirically best, that arm tends to be another near-optimal one rather than a clearly inferior one. The extra randomization introduced by softmax therefore does not change the leading Bayes regret scaling.

The bound below decomposes regret into several interpretable terms. The terms $m$ and $T\delta$ are the same coarse baseline contributions that already appear in greedy-style many-armed analyses. The logarithmic term $m(1+\log(1/\delta))$ reflects the cost of integrating over the upper tail. The summation term
\(
\frac{1}{\delta}\sum_{t=1}^T \exp(-\eta_t\delta)
\)
is the softmax-specific leakage term, controlled by the cooling schedule. Finally, the exponentially small term $T\exp(-cm\delta)$ corresponds to the bad event that too few of the $m$ arms are near-optimal.

\begin{theorem}[ASG Bayes regret, Bernoulli, linear tail]\label{thm:asg-beta1}
Assume Bernoulli rewards, and assume the prior $\Gamma$ is $1$-regular (Definition~\ref{def:beta-regular} with $\beta=1$) with regularity constants $(c_0, C_0, \varepsilon_0)$.
Fix any $\delta\in(0,\delta_0]$ with $\delta_0 \le \min\{1/8, \varepsilon_0/8\}$ (which also ensures $\delta < 1/6$, the condition required by Lemma~\ref{lem:bern-cross}).
Run ASG (Algorithm~\ref{alg:asg}) with any nonnegative nondecreasing schedule $\{\eta_t\}_{t\ge 1}$.

Then there exist constants $c,C>0$ (depending only on $c_0, C_0, \varepsilon_0$ and on the Bernoulli crossing constant in Lemma~\ref{lem:bern-cross}) such that
\begin{align}
\mathrm{BR}_{T,m}(\text{ASG})
\ &\le\ C\Bigg[
m
\ +\ T\,\delta
\ +\ m\bigl(1+\log(1/\delta)\bigr)
\ +\frac{1}{\delta}\sum_{t=1}^T \exp(-\eta_t\,\delta)
\ +\ T\,\exp(-c\,m\,\delta)
\Bigg].\label{eq:asg-main-bound}
\end{align}

In particular, choosing
\begin{equation}\label{eq:delta-choice-beta1}
\delta \ =\ \min\!\left\{\delta_0,\, A\frac{\log(T\vee 2)}{m}\right\},
\qquad
\eta_t \ =\ \frac{c_\eta}{\delta}\log(t\vee 2)\quad(c_\eta>1),
\end{equation}
with $A > 1/c$, we have
\[
\mathrm{BR}_{T,m}(\text{ASG}) \ =\ \tilde{O}\!\left(m+\frac{T}{m}\right)\,,
\]
and for $m=\Theta(\sqrt{T})$ one obtains $\mathrm{BR}_{T,m}(\text{ASG})=\tilde{O}(\sqrt{T})$.
\end{theorem}

The proof is given in Section~\ref{sec:proofs}. 

\begin{remark}[General $\beta$ (statement-level)]\label{rem:general-beta}
The linear-tail case $\beta=1$ is the cleanest to state, but the same proof strategy extends to general $\beta>0$, paralleling the corresponding generalization of the greedy benchmark in \citet{bayati2020greedy}. Repeating the proof with Lemma~\ref{lem:pdelta-lower} replaced by the general-$\beta$ analogue $p_\delta \ge c\,\delta^\beta$ (which follows from the same crossing argument together with the $\beta$-regularity bound $\Gamma([1-\delta,1]) \ge c_0 \delta^\beta$), and with the corresponding many-good-arms bound $\mathbb{P}(M(\delta) < r) \le \exp(-c m \delta^\beta)$, yields a bound of the form
\[
\mathrm{BR}_{T,m}(\text{ASG})
\ \le\ \tilde{O}\!\left(
m \ +\ T\delta
\ +\ m\cdot \mathfrak{M}_\beta(\delta)
\ +\frac{1}{\delta^\beta}\sum_{t=1}^T e^{-\eta_t\delta}
\ +\ T e^{-c m\delta^\beta}
\right),
\]
where $\mathfrak{M}_\beta(\delta)$ captures the tail-moment scaling from Lemma~\ref{lem:tail-moment-general}. In particular,
\[
\mathfrak{M}_1(\delta)=1+\log(1/\delta), \qquad
\mathfrak{M}_\beta(\delta)=O(1)\ \text{for }\beta>1, \qquad
\mathfrak{M}_\beta(\delta)=\Theta(\delta^{-(1-\beta)})\ \text{for }\beta\in(0,1).
\]
Optimizing this bound with $\delta \asymp (\log T/m)^{1/\beta}$ (paralleling the $\beta=1$ specialization) yields rates of the form $\tilde{O}\!\left(m + T(\log T/m)^{1/\beta}\right)$ for $\beta \ge 1$, matching the corresponding generalization of the greedy benchmark in \citet{bayati2020greedy}. For $\beta < 1$, the same procedure yields a similar expression with an additional $\mathfrak{M}_\beta$ contribution. The optimized rate is \emph{not} $\tilde{O}(m + T/m)$ outside $\beta=1$. The qualitative message is preserved: thicker upper tails make softmax leakage cheaper, and thinner tails give correspondingly slower rates.
\end{remark}

\section{Simulation Experiments}
\label{sec:experiments}

This section reports simulations on many-armed Bernoulli bandits: a comparison of Scheduled ASG against Pure Greedy, Thompson Sampling, and a KL-anchored variant of ASG across three prior regimes and two horizons (Section~\ref{sec:sim-baselines}), an evaluation of prior-anchored variants of Greedy and ASG that remove the forced initialization phase (Section~\ref{sec:sim-prior}), and a check that the tail index of the arm distribution drives the effect (Section~\ref{sec:sim-beta}).

\subsection{Setup}
\label{sec:sim-setup}

\paragraph{Bandit and prior regimes.}
Each trial draws $m$ arm means i.i.d.\ from $\mathrm{Uniform}([0,1])$; rewards are Bernoulli$(\mu_i)$ conditional on the drawn means. We report final cumulative regret $R_T$ averaged over independent trials, with $95\%$ confidence bands computed as $\pm 1.96 \cdot \mathrm{SE}$ across trials.

Algorithms that use a prior receive a per-arm Beta prior, which we parameterize by its mean and its weight: $\mathrm{Beta}(\lambda\nu_i,\, \lambda(1-\nu_i))$ has mean exactly $\nu_i$, and $\lambda$ acts as a pseudo-count that sets how many observations the prior is worth. We test three regimes.
\begin{itemize}
    \item \textit{Uninformative:} $\mathrm{Beta}(1,1)$ for every arm, that is, $\nu_i = 1/2$ and $\lambda = 2$. This is the setting the theorem covers directly.
    \item \textit{Informative:} $\nu_i = \mu_i$ (clipped to $[0.01, 0.99]$) with $\lambda = 4$. The prior mean of each arm equals its true mean.
    \item \textit{Misspecified:} $\nu_i = \mathrm{clip}(\mu_i + \varepsilon_i,\, 0.01,\, 0.99)$ with $\varepsilon_i \sim \mathcal{N}(0, \sigma^2)$, $\sigma = 0.3$, and $\lambda = 4$. The prior still ranks arms better than chance, but the noise is large relative to the $[0,1]$ range of the means, so individual prior means are often far off.
\end{itemize}
The three regimes separate two effects. The uninformative regime tests whether Scheduled ASG matches the theorem's rate. The other two show how algorithms that use the prior respond as prior quality drops.

\paragraph{Algorithms.}
Let $\widehat{\mu}_{i,t}$ denote the empirical mean reward of arm $i$ after $t$ rounds, let $S_i(t)$ and $N_i(t)$ denote its success count and pull count, and let $\bar\mu_{i,t} := (\lambda\nu_i + S_i(t))/(\lambda + N_i(t))$ denote its posterior mean under the arm's Beta prior; unlike the empirical mean, $\bar\mu_{i,t}$ is defined before the arm is ever pulled, where it equals $\nu_i$. Let $\eta_t = (c_\eta / \delta)\log(t \vee 2)$ with $c_\eta = 2$ and $\delta = A\log(T \vee 2)/m$, $A = 1$: the schedule form~\eqref{eq:schedule} with the theorem's choice of $\delta$, except that we drop the cap $\delta \le \delta_0$ from~\eqref{eq:delta-choice-beta1}, which would bind at the smallest arm counts.
\begin{itemize}
    \item \textit{Thompson Sampling.} At round $t$, sample $\theta_i \sim \mathrm{Beta}(\lambda\nu_i + S_i(t-1),\, \lambda(1-\nu_i) + N_i(t-1) - S_i(t-1))$ independently for each arm and pull the arm with the largest $\theta_i$.
    \item \textit{Pure Greedy} (Algorithm~\ref{alg:greedy}). Pull each arm once during a forced initialization phase of $m$ rounds, then pull $\argmax_i \widehat{\mu}_{i,t-1}$ at every subsequent round $t$.
    \item \textit{Scheduled ASG} (Algorithm~\ref{alg:asg}). Same forced initialization, then sample from softmax$(\eta_t \widehat{\mu}_{i,t-1})$.
    \item \textit{Scheduled ASG with KL anchor.} Same initialization and schedule, but sample with probability proportional to $\pi_0(i)\,\exp(\eta_t \widehat{\mu}_{i,t-1})$, where $\pi_0(i) \propto \exp(\eta_t \nu_i)$. This biases sampling toward arms the prior rates highly; under a uniform prior it reduces to Scheduled ASG.
    \item \textit{Prior-Greedy.} Score arms by the posterior mean and pull $\argmax_i \bar\mu_{i,t-1}$ from the first round, with no forced initialization.
    \item \textit{Prior-ASG.} Same score, sampled through the softmax: $A_t \sim \mathrm{softmax}(\eta_t \bar\mu_{i,t-1})$ from the first round, again with no forced initialization.
\end{itemize}
UCB1 was also run. Its regret keeps growing with $m$ across the entire grid, roughly an order of magnitude above every other algorithm, and would dominate the axes; this matches the many-armed failure mode discussed by \citet{bayati2020greedy}, and Appendix~\ref{app:ucb1-supplement} reports it separately.

\paragraph{Horizons and grids.}
We use $T = 5000$ with $m \in \{5, 10, 20, 50, 100, 200, 400\}$ and $300$ trials per point, and $T = 50000$ with $m \in \{20, 50, 100, 200, 500, 1000, 2000\}$ and $60$ trials per point. Both grids extend well below and well above $m^\star := \sqrt{T}$, the arm count at which the theorem's $m + T/m$ rate is smallest ($m^\star \approx 71$ and $m^\star \approx 224$ respectively). A vertical dashed line marks $m^\star$ in every panel.

\subsection{Scheduled ASG against the baselines}
\label{sec:sim-baselines}

Figure~\ref{fig:regret-baselines} compares Thompson Sampling, Pure Greedy, Scheduled ASG, and the KL variant. Rows are the two horizons; columns are the three prior regimes.

\paragraph{The regime the theorem describes.}
In the uninformative panels, Scheduled ASG and Pure Greedy sit on the same curve from about $m = 50$ onward at both horizons, and both follow the $m + T/m$ shape of Theorem~\ref{thm:asg-beta1} (grey dotted reference) through and past $m^\star$.  This is the empirical content of the theorem. Adding softmax randomization to greedy costs nothing in the many-armed regime.

\paragraph{Below $m^\star$, randomization helps.}
The left end of the uninformative panels shows a real gap between the two uncertainty-agnostic algorithms. At $T = 50000$, $m = 20$, Pure Greedy averages regret $1056$: it pulls each arm once, commits to the best of $20$ single Bernoulli samples, and that early winner is often the wrong arm. Scheduled ASG keeps some probability on the other arms and cuts the average to $604$. Thompson Sampling reaches $102$ at the same point, and this advantage is expected: it is the one algorithm in the figure that tracks epistemic uncertainty, and with few arms that information prevents early commitment. The comparison reverses in the many-armed range that the theorem is about. By $m = 50$ at $T = 5000$, and by $m = 100$ at $T = 50000$, the greedy-family algorithms have caught Thompson Sampling, and beyond $m^\star$ they stay ahead (at $T = 5000$, $m = 400$: Scheduled ASG $398$, Thompson Sampling $444$).

\paragraph{When the prior carries information.}
Pure Greedy and Scheduled ASG ignore the prior, so their curves barely move across the three columns (the differences at small $m$ are within the confidence bands). Thompson Sampling uses the prior through posterior sampling and noticeably improves in the informative regime. The KL variant sits between the two: the anchor $\pi_0$ injects the prior signal into sampling, but less directly than posterior sampling does. Under misspecification the anchor can hurt. At $T = 50000$, $m = 20$, the KL variant averages $1448$ against $721$ for plain Scheduled ASG, because it keeps steering probability toward arms the noisy prior overrates. Thompson Sampling stays ahead at small $m$ even with the misspecified prior, though its confidence band widens from about $\pm 15$ (informative) to $\pm 101$ (misspecified): the outcome now depends on which arms the prior error affects.

\paragraph{The shared cost at large $m$.}
In the right end of every panel, the three greedy-family curves grow linearly in $m$: the forced initialization alone costs about $m/2$ in expectation, since every arm must be pulled once and the average arm mean is $1/2$. Thompson Sampling needs no initialization, and with an informative prior it grows more slowly, but under the uninformative prior it ends slightly above the others ($2148$ at $T = 50000$, $m = 2000$). The next subsection asks whether the prior can remove this cost.

\begin{figure}[t]
    \centering
    \includegraphics[width=\textwidth]{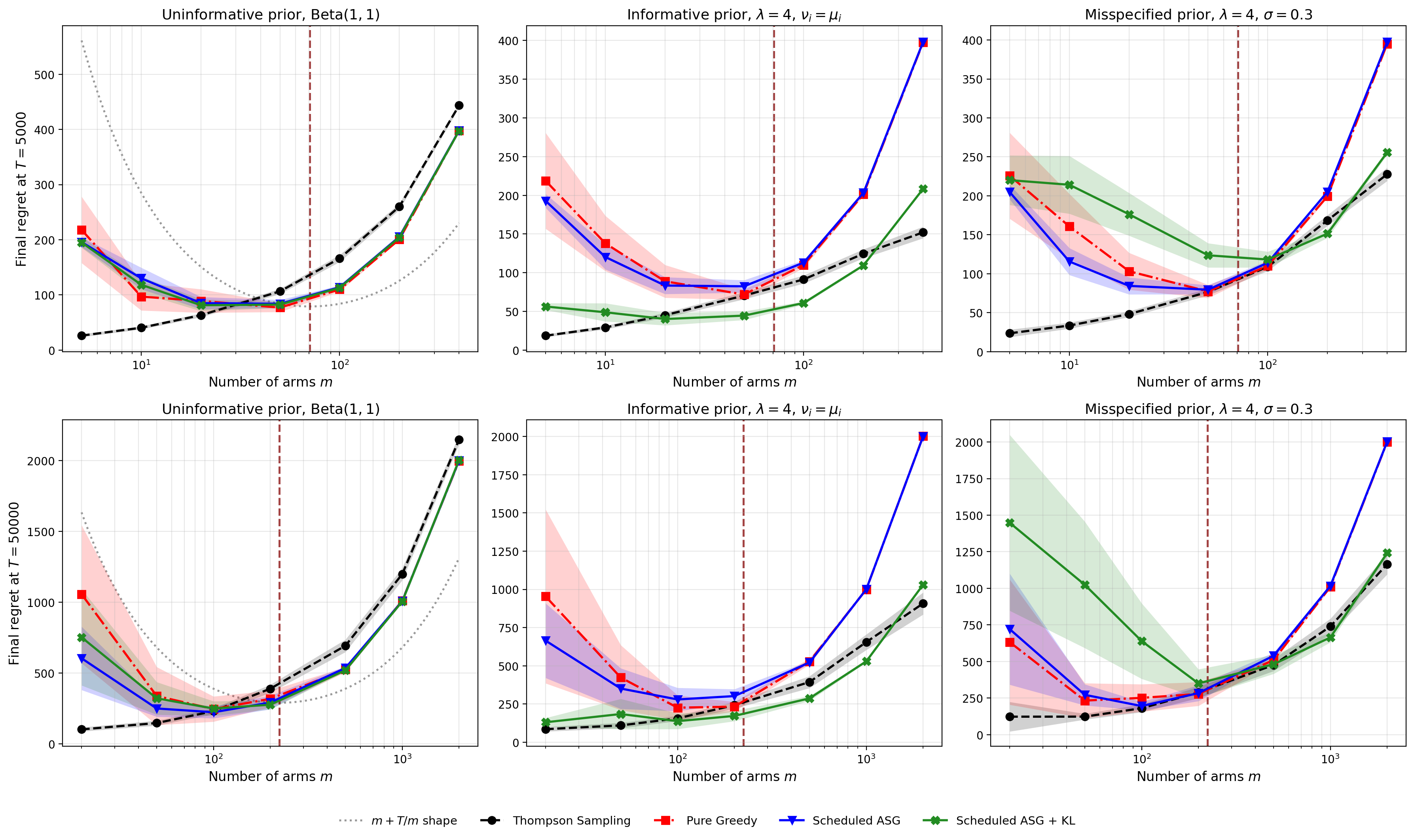}
    \caption{Final cumulative regret against the number of arms $m$ for the baseline algorithms. Top row: $T = 5000$ ($300$ trials per point). Bottom row: $T = 50000$ ($60$ trials per point). Columns: the three prior regimes of Section~\ref{sec:sim-setup}. Bands are $95\%$ confidence intervals. The vertical dashed line marks $m^\star = \sqrt{T}$; the grey dotted curve in the uninformative panels is the $m + T/m$ shape of Theorem~\ref{thm:asg-beta1}, scaled to pass through Scheduled ASG at $m = 50$ (top) and $m = 200$ (bottom). Scheduled ASG matches Pure Greedy across the many-armed range, softens Greedy's early-commitment penalty at small $m$, and moves ahead of Thompson Sampling beyond $m^\star$ in the uninformative regime.}
    \label{fig:regret-baselines}
\end{figure}

\subsection{Prior-Greedy and Prior-ASG}
\label{sec:sim-prior}

Greedy and ASG pull every arm once because the empirical mean of an unpulled arm is undefined; that is the cost the previous subsection ended on, and Thompson Sampling avoids it by starting from a prior. The same repair is available to Greedy and ASG: score arms by the posterior mean $\bar\mu_{i,t-1}$, which is defined before any pull, and drop the initialization. This gives Prior-Greedy and Prior-ASG as defined in Section~\ref{sec:sim-setup}. The structure matches LLM post-training, where the base model already assigns a probability to every completion before training starts and reweighting operates on top of that base policy rather than sampling every completion once first.

But we need to select the prior appropriately. Specifically, under the uninformative $\mathrm{Beta}(1,1)$ prior, every unpulled arm has score $1/2$. The first arm that happens to return a success jumps to $2/3$, and the $\argmax$ rule then repeats it for a long time whether or not the arm is good. Because of this, in the uninformative panels we instead run Prior-Greedy and Prior-ASG with a common optimistic prior, $\nu = 0.9$ and $\lambda = 1$. Every unpulled arm then starts at score $0.9$. A single failure drops an arm to $0.45$, below every unpulled arm, and the policy moves on to a fresh one; a single success raises it to $0.95$, and the policy stays. Appendix~\ref{app:nu-ablation} varies $\nu$ and traces the transition from the commitment failure at $\nu = 1/2$ to this rotation behavior at $\nu = 0.9$.

Figure~\ref{fig:regret-prior} adds the two variants to the comparison of Figure~\ref{fig:regret-baselines}.

\paragraph{Informative prior.}
Prior-Greedy and Prior-ASG nearly remove regret at large $m$. At $T = 5000$, $m = 400$, they average $9.8$ and $9.4$; every other algorithm is above $150$. At $T = 50000$, $m = 2000$, Prior-ASG averages $77$ and Prior-Greedy $156$, against $908$ for Thompson Sampling, $1031$ for the KL variant, and roughly $2000$ for Pure Greedy and Scheduled ASG. The prior already ranks the arms correctly, the posterior mean starts from that ranking, and with many arms whose prior means sit near the top, even the softmax spread of Prior-ASG stays on near-optimal arms.

\paragraph{Misspecified prior.}
With $\sigma = 0.3$ the prior's ranking is noisy, and at small $m$ the variants inherit its mistakes: at $T = 50000$, $m = 20$, Prior-ASG averages $588$ against $123$ for Thompson Sampling. The ordering reverses as $m$ grows. At $m = 2000$, Prior-Greedy and Prior-ASG average $177$ and $174$, against $1166$ for Thompson Sampling, $1242$ for the KL variant, and roughly $2000$ for Pure Greedy and Scheduled ASG. With many arms, even a noisy ranking places some genuinely good arms near the top, and a few pulls correct the prior's individual errors; the baselines meanwhile pay the full initialization.

\paragraph{Uninformative prior.}
The common optimistic prior is a partial substitute for real prior information, and its value depends on the horizon. At $T = 5000$, Prior-ASG tracks the baselines through the middle of the grid and is already ahead by $m = 100$, while Prior-Greedy runs above them near $m^\star$; from $m = 200$ on both variants win clearly. At $T = 50000$ the results are mixed. Prior-ASG starts level with the baselines at $m \le 50$, Prior-Greedy only at $m = 20$; both fall behind through the middle of the grid, and they move clearly ahead only at $m = 2000$. A longer horizon shrinks the initialization cost relative to the total, while a common optimistic prior carries no information about which arms are good and keeps cycling through bad ones.

\paragraph{Takeaway.}
Whenever an arm-specific prior signal exists, the prior-anchored variants win at large $m$, and they do so with the same kind of uncertainty-agnostic rule: a point-estimate score for each arm, with no posterior sample and no per-arm bonus. Appendix~\ref{app:prior-theory} outlines how the analysis behind Theorem~\ref{thm:asg-beta1} could extend to Prior-ASG under a calibrated prior, and why the common optimistic variant falls outside that argument.

\begin{figure}[t]
    \centering
    \includegraphics[width=\textwidth]{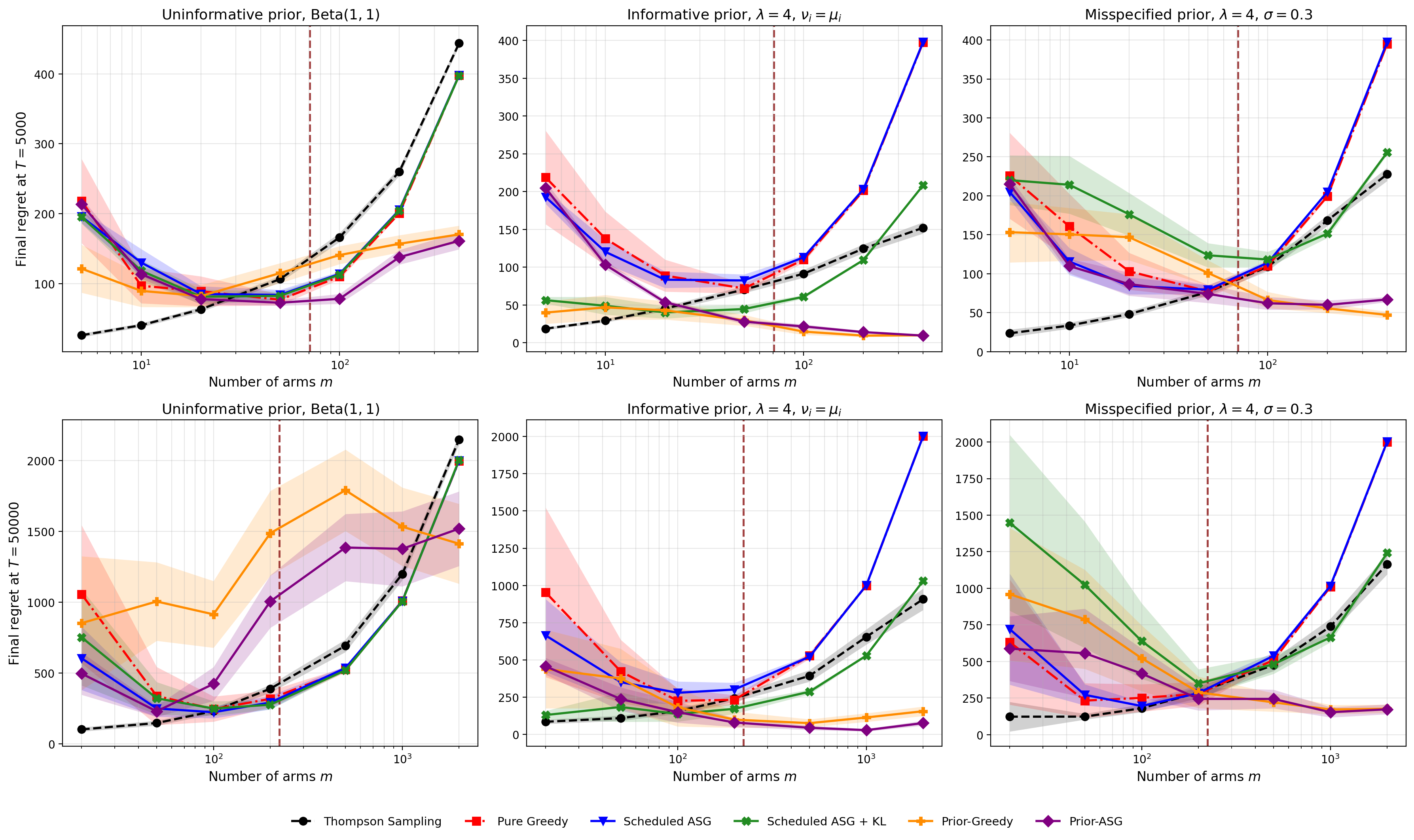}
    \caption{Prior-Greedy and Prior-ASG added to the comparison of Figure~\ref{fig:regret-baselines}; layout, bands, and the $m^\star$ line as before. In the informative and misspecified columns the two variants use the regime's per-arm prior; in the uninformative column, where no arm-specific signal exists, they use a common optimistic prior with $\nu = 0.9$ and $\lambda = 1$. With an informative prior they nearly remove regret at large $m$; with a noisy prior they lose to Thompson Sampling at small $m$ but win at large $m$; with no signal the optimistic device wins at large $m$ at $T = 5000$, while at $T = 50000$ it trails the baselines through the middle of the grid and moves clearly ahead only at $m = 2000$.}
    \label{fig:regret-prior}
\end{figure}

\subsection{Varying the tail index}
\label{sec:sim-beta}

The analysis attributes the good behavior of the uncertainty-agnostic algorithms to the thick upper tail of the arm distribution, so we vary that thickness directly. Drawing the arm means from $\mathrm{Beta}(1, \beta)$ gives $\Gamma([1-\varepsilon, 1]) = \varepsilon^\beta$ exactly, so $\beta$ is the regularity index of Definition~\ref{def:beta-regular}, and $\beta = 1$ recovers the $\mathrm{Uniform}([0,1])$ draws used above. We run Pure Greedy and Scheduled ASG (same schedule as in Section~\ref{sec:sim-setup}) over $\beta \in \{0.3, 0.5, 0.7, 1, 1.5, 2, 3, 5\}$ and $m \in \{50, 100, 200, 400\}$ at $T = 5000$, with $100$ trials per point.

Figure~\ref{fig:regret-beta} shows the result. Regret rises as the tail thins in every panel, consistent with Remark~\ref{rem:general-beta}, and the rise is far steeper when arms are scarce: from $\beta = 0.3$ to $\beta = 5$, Pure Greedy's regret grows about ninefold at $m = 50$ ($43$ to $389$) but only $2.3$-fold at $m = 400$ ($266$ to $613$). This matches the mechanism behind the theorem. A thick tail supplies many near-optimal arms, which is what makes acting on point estimates alone safe; when the tail is thin and arms are few, near-optimal arms are rare and the same rule becomes costly.

\begin{figure}[t]
    \centering
    \includegraphics[width=\textwidth]{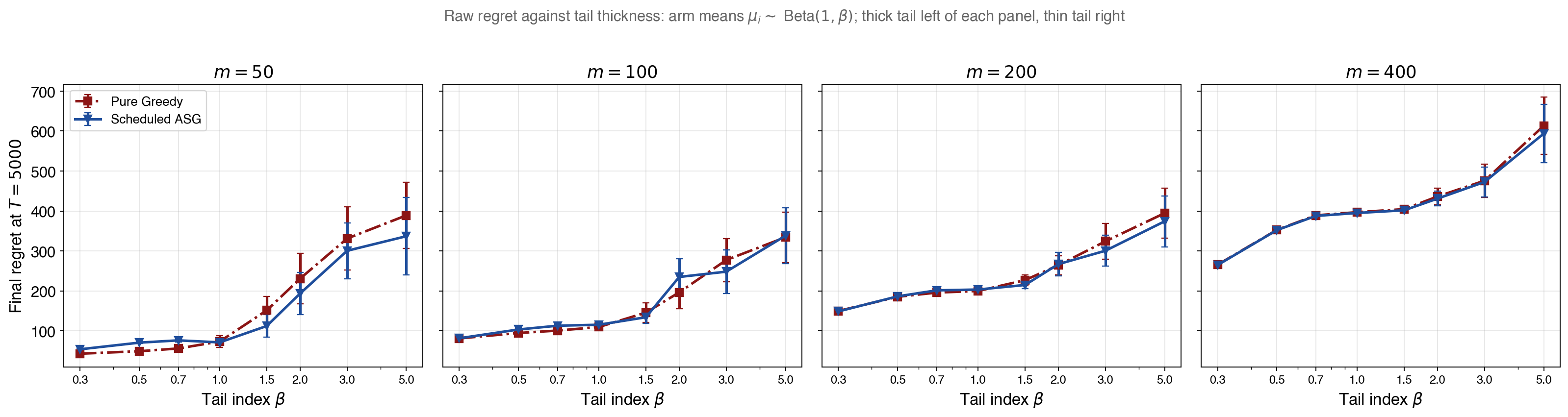}
    \caption{Final regret of Pure Greedy and Scheduled ASG at $T = 5000$ with arm means drawn from $\mathrm{Beta}(1, \beta)$, for which $\Gamma([1-\varepsilon, 1]) = \varepsilon^\beta$; $\beta = 1$ is the uninformative setting of the earlier figures. One panel per arm count; bars are $95\%$ confidence intervals over $100$ trials. Regret grows as the tail thins, and much faster at small $m$.}
    \label{fig:regret-beta}
\end{figure}

\section{Discussion}
\label{sec:discussion}

This paper studies whether annealed softmax over empirical mean rewards, despite ignoring epistemic uncertainty, can achieve near-optimal Bayes regret in the many-armed Bayesian Bernoulli bandit. Theorem~\ref{thm:asg-beta1} answers this positively in the $1$-regular case of the upper-tail framework introduced in prior literature for empirical-mean greedy: ASG attains the same near-optimal rate, up to logarithmic factors.

The mechanism is the abundance of near-optimal arms under $\beta$-regularity: with high probability, many arms maintain empirical means close to the optimum throughout learning. When softmax samples an arm other than the empirically best, that arm tends to be another near-optimal one rather than a clearly inferior one, so the price of softmax randomization is paid on near-optimal arms and contributes little to regret. Our proof builds on the analysis strategy of \citet{bayati2020greedy} for greedy, adding control of a softmax leakage term that quantifies the cost of placing some probability away from the empirical argmax.

Two limitations remain. First, the model is narrow: we study i.i.d.\ Bayesian Bernoulli bandits, with no context, no shared structure across actions, and no sequential state dynamics. Second, the guarantees are Bayesian rather than minimax, and rely on upper-tail regularity of the prior; if near-optimal arms are rare, the mechanism weakens.

The most direct extension is to structured or contextual action spaces, and to sequential RL settings where ``many near-optimal actions'' might be replaced by many near-optimal trajectories or completions. More broadly, an open question is to characterize the boundary of the phenomenon: precisely when does upper-tail abundance make uncertainty-agnostic reweighting effective, and when do the classical Boltzmann failure modes re-emerge?

\bibliographystyle{plainnat}
\bibliography{main}

@article{thompson1933likelihood,
  title        = {On the likelihood that one unknown probability exceeds another in view of the evidence of two samples},
  author       = {Thompson, William R.},
  journal      = {Biometrika},
  volume       = {25},
  number       = {3/4},
  pages        = {285--294},
  year         = {1933},
  doi          = {10.2307/2332286},
  url          = {https://doi.org/10.2307/2332286}
}

@inproceedings{agrawal2012analysis,
  title        = {Analysis of {T}hompson Sampling for the Multi-armed Bandit Problem},
  author       = {Agrawal, Shipra and Goyal, Navin},
  booktitle    = {Proceedings of the 25th Annual Conference on Learning Theory (COLT)},
  series       = {Proceedings of Machine Learning Research},
  volume       = {23},
  pages        = {39.1--39.26},
  year         = {2012},
  publisher    = {PMLR},
  url          = {https://proceedings.mlr.press/v23/agrawal12.html}
}

@article{lai1985adaptive,
  title        = {Asymptotically efficient adaptive allocation rules},
  author       = {Lai, Tze Leung and Robbins, Herbert},
  journal      = {Advances in Applied Mathematics},
  volume       = {6},
  number       = {1},
  pages        = {4--22},
  year         = {1985},
  doi          = {10.1016/0196-8858(85)90002-8},
  url          = {https://doi.org/10.1016/0196-8858(85)90002-8}
}

@article{auer2002finite,
  title        = {Finite-time analysis of the multiarmed bandit problem},
  author       = {Auer, Peter and Cesa-Bianchi, Nicolo and Fischer, Paul},
  journal      = {Machine Learning},
  volume       = {47},
  number       = {2-3},
  pages        = {235--256},
  year         = {2002},
  doi          = {10.1023/A:1013689704352},
  url          = {https://doi.org/10.1023/A:1013689704352}
}

@inproceedings{audibert2007tuning,
  title        = {Tuning bandit algorithms in stochastic environments},
  author       = {Audibert, Jean-Yves and Munos, R{\'e}mi and Szepesv{\'a}ri, Csaba},
  booktitle    = {International Conference on Algorithmic Learning Theory (ALT)},
  series       = {Lecture Notes in Computer Science},
  volume       = {4754},
  pages        = {150--165},
  year         = {2007},
  organization = {Springer},
  doi          = {10.1007/978-3-540-75225-7_15},
  url          = {https://doi.org/10.1007/978-3-540-75225-7_15}
}

@article{bubeck2012regret,
  title        = {Regret analysis of stochastic and nonstochastic multi-armed bandit problems},
  author       = {Bubeck, S{\'e}bastien and Cesa-Bianchi, Nicolo},
  journal      = {Foundations and Trends{\textregistered} in Machine Learning},
  volume       = {5},
  number       = {1},
  pages        = {1--122},
  year         = {2012},
  doi          = {10.1561/2200000024},
  url          = {https://doi.org/10.1561/2200000024}
}

@book{lattimore2020bandit,
  title        = {Bandit Algorithms},
  author       = {Lattimore, Tor and Szepesv{\'a}ri, Csaba},
  publisher    = {Cambridge University Press},
  year         = {2020},
  doi          = {10.1017/9781108571401},
  url          = {https://doi.org/10.1017/9781108571401}
}

@article{gittins1979bandit,
  title        = {Bandit processes and dynamic allocation indices},
  author       = {Gittins, John C.},
  journal      = {Journal of the Royal Statistical Society: Series B (Methodological)},
  volume       = {41},
  number       = {2},
  pages        = {148--164},
  year         = {1979},
  doi          = {10.1111/j.2517-6161.1979.tb01068.x},
  url          = {https://doi.org/10.1111/j.2517-6161.1979.tb01068.x}
}

@article{kaufmann2018bayesian,
  title        = {On {B}ayesian index policies for sequential resource allocation},
  author       = {Kaufmann, Emilie},
  journal      = {The Annals of Statistics},
  volume       = {46},
  number       = {2},
  pages        = {842--865},
  year         = {2018},
  doi          = {10.1214/17-AOS1569},
  url          = {https://doi.org/10.1214/17-AOS1569}
}

@inproceedings{abbasi2011oful,
  title        = {Improved algorithms for linear stochastic bandits},
  author       = {Abbasi-Yadkori, Yasin and P{\'a}l, D{\'a}vid and Szepesv{\'a}ri, Csaba},
  booktitle    = {Advances in Neural Information Processing Systems (NeurIPS)},
  pages        = {2312--2320},
  year         = {2011},
  url          = {https://proceedings.neurips.cc/paper/2011/hash/e1d5be1c7f2f456670de3d53c7b54f4a-Abstract.html}
}

@inproceedings{russo2014learning,
  title        = {Learning to optimize via information-directed sampling},
  author       = {Russo, Daniel and Van Roy, Benjamin},
  booktitle    = {Advances in Neural Information Processing Systems (NeurIPS)},
  pages        = {1583--1591},
  year         = {2014},
  url          = {https://papers.nips.cc/paper/5463-learning-to-optimize-via-information-directed-sampling}
}

@article{russo2014posterior,
  title        = {Learning to optimize via posterior sampling},
  author       = {Russo, Daniel and Van Roy, Benjamin},
  journal      = {Mathematics of Operations Research},
  volume       = {39},
  number       = {4},
  pages        = {1221--1243},
  year         = {2014},
  doi          = {10.1287/moor.2014.0650},
  url          = {https://doi.org/10.1287/moor.2014.0650}
}

@article{russo2016info,
  title        = {An information-theoretic analysis of {T}hompson sampling},
  author       = {Russo, Daniel and Van Roy, Benjamin},
  journal      = {Journal of Machine Learning Research},
  volume       = {17},
  number       = {68},
  pages        = {1--30},
  year         = {2016},
  url          = {https://jmlr.org/papers/v17/14-087.html}
}

@misc{russo2018satisficing,
  title        = {Satisficing in time-sensitive bandit learning},
  author       = {Russo, Daniel and Van Roy, Benjamin},
  year         = {2018},
  eprint       = {1803.02855},
  archivePrefix= {arXiv},
  primaryClass = {stat.ML},
  url          = {https://arxiv.org/abs/1803.02855}
}

@inproceedings{cesabianchi2017boltzmann,
  title        = {Boltzmann Exploration Done Right},
  author       = {Cesa-Bianchi, Nicol{\`o} and Gentile, Claudio and Lugosi, G{\'a}bor and Neu, Gergely},
  booktitle    = {Advances in Neural Information Processing Systems (NeurIPS)},
  year         = {2017},
  eprint       = {1705.10257},
  archivePrefix= {arXiv},
  primaryClass = {cs.LG},
  url          = {https://arxiv.org/abs/1705.10257}
}

@inproceedings{bayati2020greedy,
  title        = {Unreasonable Effectiveness of Greedy Algorithms in Multi-Armed Bandit with Many Arms},
  author       = {Bayati, Mohsen and Hamidi, Nima and Johari, Ramesh and Khosravi, Khashayar},
  booktitle    = {Advances in Neural Information Processing Systems (NeurIPS)},
  year         = {2020},
  eprint       = {2002.10121},
  archivePrefix= {arXiv},
  primaryClass = {cs.LG},
  url          = {https://proceedings.neurips.cc/paper/2020/hash/12d16adf4a9355513f9d574b76087a08-Abstract.html}
}

@article{berry1997infinitely,
  title        = {Bandit problems with infinitely many arms},
  author       = {Berry, Donald A. and Chen, Robert W. and Zame, Alan and Heath, David C. and Shepp, Larry A.},
  journal      = {The Annals of Statistics},
  volume       = {25},
  number       = {5},
  pages        = {2103--2116},
  year         = {1997},
  doi          = {10.1214/aos/1069362389},
  url          = {https://doi.org/10.1214/aos/1069362389}
}

@inproceedings{wang2009algorithms,
  title        = {Algorithms for infinitely many-armed bandits},
  author       = {Wang, Yizao and Audibert, Jean-Yves and Munos, R{\'e}mi},
  booktitle    = {Advances in Neural Information Processing Systems 21 (NeurIPS)},
  pages        = {1729--1736},
  year         = {2008},
  url          = {https://proceedings.neurips.cc/paper/2008/hash/49ae49a23f67c759bf4fc791ba842aa2-Abstract.html}
}

@inproceedings{bonald2013twotarget,
  title        = {Two-target algorithms for infinite-armed bandits with {B}ernoulli rewards},
  author       = {Bonald, Thomas and Prouti{\`e}re, Alexandre},
  booktitle    = {Advances in Neural Information Processing Systems (NeurIPS)},
  pages        = {2184--2192},
  year         = {2013},
  url          = {https://proceedings.neurips.cc/paper/2013/hash/fc2c7c47b918d0c2d792a719dfb602ef-Abstract.html}
}

@inproceedings{carpentier2015simple,
  title        = {Simple regret for infinitely many armed bandits},
  author       = {Carpentier, Alexandra and Valko, Michal},
  booktitle    = {International Conference on Machine Learning (ICML)},
  series       = {Proceedings of Machine Learning Research},
  volume       = {37},
  pages        = {1133--1141},
  year         = {2015},
  publisher    = {PMLR},
  url          = {https://proceedings.mlr.press/v37/carpentier15.html}
}

@inproceedings{chaudhuri2018quantile,
  title        = {Quantile-regret minimisation in infinitely many-armed bandits},
  author       = {Chaudhuri, Arghya Roy and Kalyanakrishnan, Shivaram},
  booktitle    = {Conference on Uncertainty in Artificial Intelligence (UAI)},
  pages        = {425--434},
  year         = {2018},
  url          = {http://auai.org/uai2018/proceedings/papers/169.pdf}
}

@article{bastani2020mostly,
  title        = {Mostly Exploration-Free Algorithms for Contextual Bandits},
  author       = {Bastani, Hamsa and Bayati, Mohsen and Khosravi, Khashayar},
  journal      = {Management Science},
  volume       = {67},
  number       = {3},
  pages        = {1329--1349},
  year         = {2020},
  doi          = {10.1287/mnsc.2020.3605},
  url          = {https://doi.org/10.1287/mnsc.2020.3605}
}

@inproceedings{kannan2018smoothed,
  title        = {A smoothed analysis of the greedy algorithm for the linear contextual bandit problem},
  author       = {Kannan, Sampath and Morgenstern, Jamie H. and Roth, Aaron and Waggoner, Bo and Wu, Zhiwei Steven},
  booktitle    = {Advances in Neural Information Processing Systems (NeurIPS)},
  pages        = {2227--2236},
  year         = {2018},
  url          = {https://proceedings.neurips.cc/paper/2018/hash/2cfd4560539f887a5e420412b370b361-Abstract.html}
}

@inproceedings{hao2020adaptive,
  title        = {Adaptive exploration in linear contextual bandit},
  author       = {Hao, Botao and Lattimore, Tor and Szepesv{\'a}ri, Csaba},
  booktitle    = {International Conference on Artificial Intelligence and Statistics (AISTATS)},
  series       = {Proceedings of Machine Learning Research},
  volume       = {108},
  pages        = {3536--3545},
  year         = {2020},
  organization = {PMLR},
  url          = {https://proceedings.mlr.press/v108/hao20b.html}
}

@inproceedings{raghavan2018externalities,
  title        = {The externalities of exploration and how data diversity helps exploitation},
  author       = {Raghavan, Manish and Slivkins, Aleksandrs and Vaughan, Jennifer Wortman and Wu, Zhiwei Steven},
  booktitle    = {Conference on Learning Theory (COLT)},
  series       = {Proceedings of Machine Learning Research},
  volume       = {75},
  pages        = {1724--1738},
  year         = {2018},
  publisher    = {PMLR},
  eprint       = {1806.00543},
  archivePrefix= {arXiv},
  primaryClass = {cs.LG},
  url          = {https://proceedings.mlr.press/v75/raghavan18a.html}
}

@misc{shao2024deepseekmath,
  title        = {DeepSeekMath: Pushing the Limits of Mathematical Reasoning in Open Language Models},
  author       = {Shao, Zhihong and Wang, Peiyi and Zhu, Qihao and Xu, Runxin and Song, Junxiao and Bi, Xiao and Zhang, Haowei and Zhang, Mingchuan and Li, Y.K. and Wu, Y. and Guo, Daya},
  year         = {2024},
  eprint       = {2402.03300},
  archivePrefix= {arXiv},
  primaryClass = {cs.CL},
  url          = {https://arxiv.org/abs/2402.03300}
}

@misc{deepseekai2025r1,
  title        = {DeepSeek-R1: Incentivizing Reasoning Capability in {LLM}s via Reinforcement Learning},
  author       = {{DeepSeek-AI}},
  year         = {2025},
  eprint       = {2501.12948},
  archivePrefix= {arXiv},
  primaryClass = {cs.CL},
  url          = {https://arxiv.org/abs/2501.12948},
  note         = {Published in Nature 645, 633--638 (2025), doi:10.1038/s41586-025-09422-z}
}

@misc{mroueh2025grpo,
  title        = {Reinforcement Learning with Verifiable Rewards: {GRPO}'s Effective Loss, Dynamics, and Success Amplification},
  author       = {Mroueh, Youssef},
  year         = {2025},
  eprint       = {2503.06639},
  archivePrefix= {arXiv},
  primaryClass = {cs.LG},
  url          = {https://arxiv.org/abs/2503.06639}
}

@article{yue2025rlvr,
  title        = {Does Reinforcement Learning Really Incentivize Reasoning Capacity in {LLM}s Beyond the Base Model?},
  author       = {Yue, Yang and Chen, Zhiqi and Lu, Rui and Zhao, Andrew and Wang, Zhaokai and Song, Shiji and Huang, Gao},
  year         = {2025},
  eprint       = {2504.13837},
  archivePrefix= {arXiv},
  primaryClass = {cs.AI},
  url          = {https://arxiv.org/abs/2504.13837}
}

@misc{wen2025implicit,
  title        = {Reinforcement Learning with Verifiable Rewards Implicitly Incentivizes Correct Reasoning in Base {LLM}s},
  author       = {Wen, Xumeng and Liu, Zihan and Zheng, Shun and Ye, Shengyu and Wu, Zhirong and Wang, Yang and Xu, Zhijian and Liang, Xiao and Li, Junjie and Miao, Ziming and Bian, Jiang and Yang, Mao},
  year         = {2025},
  eprint       = {2506.14245},
  archivePrefix= {arXiv},
  primaryClass = {cs.AI},
  url          = {https://arxiv.org/abs/2506.14245}
}

@misc{liu2025prorl,
  title        = {ProRL: Prolonged Reinforcement Learning Expands Reasoning Boundaries in Large Language Models},
  author       = {Liu, Mingjie and Diao, Shizhe and Lu, Ximing and Hu, Jian and Dong, Xin and Choi, Yejin and Kautz, Jan and Dong, Yi},
  year         = {2025},
  eprint       = {2505.24864},
  archivePrefix= {arXiv},
  primaryClass = {cs.CL},
  url          = {https://arxiv.org/abs/2505.24864}
}

@misc{liu2025r1zero,
  title        = {Understanding {R}1-Zero-Like Training: A Critical Perspective},
  author       = {Liu, Zichen and Chen, Changyu and Li, Wenjun and Qi, Penghui and Pang, Tianyu and Du, Chao and Lee, Wee Sun and Lin, Min},
  year         = {2025},
  eprint       = {2503.20783},
  archivePrefix= {arXiv},
  primaryClass = {cs.LG},
  url          = {https://arxiv.org/abs/2503.20783}
}

@misc{yu2025dapo,
  title        = {{DAPO}: An Open-Source {LLM} Reinforcement Learning System at Scale},
  author       = {Yu, Qiying and Zhang, Zheng and Zhu, Ruofei and Yuan, Yufeng and Zuo, Xiaochen and Yue, Yu and Dai, Weinan and Fan, Tiantian and Liu, Gaohong and Liu, Lingjun and others},
  year         = {2025},
  eprint       = {2503.14476},
  archivePrefix= {arXiv},
  primaryClass = {cs.LG},
  url          = {https://arxiv.org/abs/2503.14476}
}

@misc{cui2025entropy,
  title        = {The Entropy Mechanism of Reinforcement Learning for Reasoning Language Models},
  author       = {Cui, Ganqu and Zhang, Yuchen and Chen, Jiacheng and Yuan, Lifan and Wang, Zhi and Zuo, Yuxin and Li, Haozhan and Fan, Yuchen and Chen, Huayu and Chen, Weize and Liu, Zhiyuan and Peng, Hao and Bai, Lei and Ouyang, Wanli and Cheng, Yu and Zhou, Bowen and Ding, Ning},
  year         = {2025},
  eprint       = {2505.22617},
  archivePrefix= {arXiv},
  primaryClass = {cs.CL},
  url          = {https://arxiv.org/abs/2505.22617}
}

@misc{zeng2025shrinkage,
  title        = {Shrinking the Variance: Shrinkage Baselines for Reinforcement Learning with Verifiable Rewards},
  author       = {Zeng, Guanning and Zhou, Zhaoyi and Arora, Daman and Zanette, Andrea},
  year         = {2025},
  eprint       = {2511.03710},
  archivePrefix= {arXiv},
  primaryClass = {cs.LG},
  url          = {https://arxiv.org/abs/2511.03710}
}

@misc{peng2025simko,
  title        = {SimKO: Simple Pass@{K} Policy Optimization},
  author       = {Peng, Ruotian and Ren, Yi and Yu, Zhouliang and Liu, Weiyang and Wen, Yandong},
  year         = {2025},
  eprint       = {2510.14807},
  archivePrefix= {arXiv},
  primaryClass = {cs.AI},
  url          = {https://arxiv.org/abs/2510.14807}
}

@misc{yu2025passk,
  title        = {Pass@{k} Metric for {RLVR}: A Diagnostic Tool of Exploration, But Not an Objective},
  author       = {Yu, Yang},
  year         = {2025},
  eprint       = {2511.16231},
  archivePrefix= {arXiv},
  primaryClass = {cs.LG},
  url          = {https://arxiv.org/abs/2511.16231}
}

@misc{chen2021evaluating,
  title        = {Evaluating Large Language Models Trained on Code},
  author       = {Chen, Mark and Tworek, Jerry and Jun, Heewoo and Yuan, Qiming and Pinto, Henrique Ponde de Oliveira and Kaplan, Jared and Edwards, Harri and Burda, Yuri and Joseph, Nicholas and Brockman, Greg and others},
  year         = {2021},
  eprint       = {2107.03374},
  archivePrefix= {arXiv},
  primaryClass = {cs.LG},
  url          = {https://arxiv.org/abs/2107.03374}
}

@misc{ouyang2022training,
  title        = {Training language models to follow instructions with human feedback},
  author       = {Ouyang, Long and Wu, Jeff and Jiang, Xu and Almeida, Diogo and Wainwright, Carroll L. and Mishkin, Pamela and Zhang, Chong and Agarwal, Sandhini and Slama, Katarina and Ray, Alex and Schulman, John and Hilton, Jacob and Kelton, Fraser and Miller, Luke and Simens, Maddie and Askell, Amanda and Welinder, Peter and Christiano, Paul and Leike, Jan and Lowe, Ryan},
  year         = {2022},
  eprint       = {2203.02155},
  archivePrefix= {arXiv},
  primaryClass = {cs.CL},
  url          = {https://arxiv.org/abs/2203.02155},
  note         = {Advances in Neural Information Processing Systems 35 (NeurIPS 2022)}
}

@misc{schulman2017ppo,
  title        = {Proximal Policy Optimization Algorithms},
  author       = {Schulman, John and Wolski, Filip and Dhariwal, Prafulla and Radford, Alec and Klimov, Oleg},
  year         = {2017},
  eprint       = {1707.06347},
  archivePrefix= {arXiv},
  primaryClass = {cs.LG},
  url          = {https://arxiv.org/abs/1707.06347}
}

@inproceedings{schulman2015trpo,
  title        = {Trust Region Policy Optimization},
  author       = {Schulman, John and Levine, Sergey and Moritz, Philipp and Jordan, Michael I. and Abbeel, Pieter},
  booktitle    = {International Conference on Machine Learning (ICML)},
  series       = {Proceedings of Machine Learning Research},
  volume       = {37},
  pages        = {1889--1897},
  year         = {2015},
  publisher    = {PMLR},
  url          = {https://proceedings.mlr.press/v37/schulman15.html}
}

@misc{zhao2025sharp,
  title        = {Sharp Analysis for {KL}-Regularized Contextual Bandits and {RLHF}},
  author       = {Zhao, Heyang and Ye, Chenlu and Gu, Quanquan and Zhang, Tong},
  year         = {2025},
  eprint       = {2411.04625},
  archivePrefix= {arXiv},
  primaryClass = {cs.LG},
  url          = {https://openreview.net/forum?id=TE63KPCXWt},
  note         = {NeurIPS 2025 poster; arXiv:2411.04625}
}

@misc{zhang2025survey,
  title        = {A Survey of Reinforcement Learning for Large Reasoning Models},
  author       = {Zhang, Kaiyan and Zuo, Yuxin and He, Bingxiang and Sun, Youbang and Liu, Runze and Jiang, Che and Fan, Yuchen and Tian, Kai and Jia, Guoli and Li, Pengfei and Fu, Yu and Lv, Xingtai and Zhang, Yuchen and Zeng, Sihang and Qu, Shang and Li, Haozhan and Wang, Shijie and Wang, Yuru and Long, Xinwei and Liu, Fangfu and Xu, Xiang and Ma, Jiaze and Zhu, Xuekai and Hua, Ermo and Liu, Yihao and Li, Zonglin and Chen, Huayu and Qu, Xiaoye and Li, Yafu and Chen, Weize and Yuan, Zhenzhao and Gao, Junqi and Li, Dong and Ma, Zhiyuan and Cui, Ganqu and Liu, Zhiyuan and Qi, Biqing and Ding, Ning and Zhou, Bowen},
  year         = {2025},
  eprint       = {2509.08827},
  archivePrefix= {arXiv},
  primaryClass = {cs.CL},
  url          = {https://arxiv.org/abs/2509.08827}
}

@misc{stojanovski2025reasoninggym,
  title        = {Reasoning Gym: Reasoning Environments for Reinforcement Learning with Verifiable Rewards},
  author       = {Stojanovski, Zafir and Stanley, Oliver and Sharratt, Joe and Jones, Richard and Adefioye, Abdulhakeem and Kaddour, Jean and K{\"o}pf, Andreas},
  year         = {2025},
  eprint       = {2505.24760},
  archivePrefix= {arXiv},
  primaryClass = {cs.LG},
  url          = {https://arxiv.org/abs/2505.24760},
  note         = {NeurIPS 2025 Spotlight}
}

\appendix

\section{Proofs}
\label{sec:proofs}

The proof technique follows the many-armed analysis of \cite{bayati2020greedy}, adapted to control the softmax leakage probability. Even when the score vector is already informative, Boltzmann exploration still assigns some probability mass to inferior arms \citep{cesabianchi2017boltzmann}. Under a $1$-regular prior and many available arms, we show that the presence of many persistently good arms suppresses this leakage enough to recover the many-armed regret rate.

Throughout, $\mathcal{F}_t$ denotes the history $\sigma$-field generated by actions and observations up to time $t$.

\subsection{A Bernoulli ``never-crossing'' tail event}

\begin{definition}[Never-crossing probability $q_\theta(\mu)$]\label{def:qtheta}
Fix $\theta\in(0,1)$ and $\mu \in [0, 1]$.
Let $\{X_s\}_{s\ge 1}$ be i.i.d.\ Bernoulli$(\mu)$ and let $\widehat{\mu}(n):=\frac1n\sum_{s=1}^n X_s$.
Define
\[
q_\theta(\mu)\ :=\ \mathbb{P}\big(\widehat{\mu}(n)>\theta \ \text{for all } n\ge 1\big),
\]
with the convention $q_\theta(\mu) = 0$ for $\mu \le \theta$.
\end{definition}

\begin{lemma}[Bernoulli crossing bound {\citep[Lemma~4.2]{bayati2020greedy}}]\label{lem:bern-cross}
Let $X_s\sim\mathrm{Bernoulli}(\mu)$ i.i.d.\ and fix $\theta>2/3$.
If $\mu\ge (1+\theta)/2$, then
\[
q_\theta(\mu)\ \ge\ C_{\mathrm{Bern}}\ :=\ \frac{e^{-1/2}}{3}.
\]
\end{lemma}

\subsection{Many always-good arms under a $\beta$-regular tail}
We now formalize the event that the prior produces many ``always-good'' arms, which will be used to control softmax leakage.

Fix $\delta\in(0,1/8]$ and define two thresholds
\[
\theta := 1-2\delta,
\qquad
\theta' := 1-3\delta.
\]
For each arm $i\in[m]$, define the \emph{always-good} event
\begin{equation}\label{eq:Gi-def}
\mathcal{G}_i(\delta)\ :=\ \Big\{\mu_i \ge 1-\delta\Big\}\ \cap\ \Big\{\widehat{\mu}_i(n)>\theta\ \ \forall n\ge 1\Big\},
\end{equation}
where $\widehat{\mu}_i(n)$ denotes the empirical mean of the \emph{first $n$ rewards} of arm $i$ (a property of the arm's reward sequence, independent of the policy).

Let
\[
M(\delta)\ :=\ \sum_{i=1}^m \mathbf{1}\{\mathcal{G}_i(\delta)\}
\]
be the number of always-good arms.

\begin{lemma}[Expected mass of always-good arms]\label{lem:pdelta-lower}
Assume $\Gamma$ is $1$-regular (Definition~\ref{def:beta-regular} with $\beta=1$).
There exists $\delta_0>0$ and a constant $c>0$ such that for all $\delta\in(0,\delta_0]$,
\[
p_\delta\ :=\ \mathbb{P}\big(\mathcal{G}_i(\delta)\big)
\ =\ \mathbb{E}_{\mu\sim\Gamma}\big[\mathbf{1}\{\mu\ge 1-\delta\}\,q_{1-2\delta}(\mu)\big]
\ \ge\ c\,\delta.
\]
\end{lemma}

\begin{proof}
By Lemma~\ref{lem:bern-cross}, for $\delta<1/6$ and $\mu\ge 1-\delta$, we have
$q_{1-2\delta}(\mu)\ge C_{\mathrm{Bern}}$ since $1-\delta\ge (1+(1-2\delta))/2$.
Therefore
\[
p_\delta \ \ge\ C_{\mathrm{Bern}}\cdot \mathbb{P}(\mu\ge 1-\delta).
\]
By $1$-regularity, $\mathbb{P}(\mu\ge 1-\delta)\ge c_0\,\delta$ for all sufficiently small $\delta$, hence $p_\delta\ge C_{\mathrm{Bern}}c_0\,\delta$.
\end{proof}

\begin{lemma}[Many always-good arms with high probability]\label{lem:many-good-arms}
Under the assumptions of Lemma~\ref{lem:pdelta-lower}, there exist constants $c,c'>0$ such that for all sufficiently small $\delta$,
\[
\mathbb{P}\Big(M(\delta) \ \ge\ r(\delta)\Big) \ \ge\ 1 - \exp(-c'\,m\,\delta),
\qquad
r(\delta)\ :=\ \max\!\left\{1,\left\lfloor \frac{m p_\delta}{2}\right\rfloor\right\}.
\]
\end{lemma}

\begin{proof}
Across arms, the pairs $(\mu_i, (X_{i,s})_{s\ge 1})$ are i.i.d., so the indicators $\mathbf{1}\{\mathcal{G}_i(\delta)\}$ are i.i.d.\ Bernoulli with mean $p_\delta$.
Thus $M(\delta)\sim\mathrm{Binomial}(m,p_\delta)$.

If $m p_\delta\ge 2$, then $r(\delta)=\lfloor m p_\delta/2\rfloor$ and a multiplicative Chernoff bound gives
\[
\mathbb{P}\big(M(\delta) < r(\delta)\big)
\le
\mathbb{P}\big(M(\delta) < m p_\delta/2\big)
\le
\exp(- m p_\delta/8).
\]
If $m p_\delta<2$, then $r(\delta)=1$ and
\[
\mathbb{P}\big(M(\delta) < r(\delta)\big)
=
\mathbb{P}\big(M(\delta)=0\big)
=
(1-p_\delta)^m
\le
\exp(-m p_\delta).
\]
In either case,
\[
\mathbb{P}\big(M(\delta) < r(\delta)\big)\le \exp(-c\,m p_\delta).
\]
Lemma~\ref{lem:pdelta-lower} gives $p_\delta\ge c' \delta$, which yields the stated bound after adjusting constants.
\end{proof}

\subsection{A key inequality: softmax leakage controlled by $r$ good arms}
We next bound how often a suboptimal arm can be sampled when many good arms maintain a score margin.

\begin{lemma}[Pointwise leakage bound for softmax]\label{lem:softmax-pointwise}
Fix time $t$, scores $\{s_j\}_{j=1}^m\subset\mathbb{R}$, an inverse temperature $\eta_t \ge 0$, and an integer $r \in \{1, \dots, m\}$.
Let $p_t(i)\propto \exp(\eta_t s_i)$.
If there is a subset $\mathcal{J}\subset[m]$ with $|\mathcal{J}|=r$ such that $s_j\ge s^\star$ for all $j\in\mathcal{J}$, then for any $i$,
\[
p_t(i)\ \le\ \frac{\exp(\eta_t s_i)}{r\,\exp(\eta_t s^\star)}
\ =\ \frac{1}{r}\exp\big(-\eta_t (s^\star - s_i)\big).
\]
In particular, if $s^\star - s_i \ge \Delta>0$, then $p_t(i)\le \frac{1}{r}e^{-\eta_t \Delta}$.
\end{lemma}

\begin{proof}
Immediate from $\sum_{j=1}^m \exp(\eta_t s_j) \ge \sum_{j\in\mathcal{J}} \exp(\eta_t s_j)\ge r\exp(\eta_t s^\star)$.
\end{proof}

\subsection{Bounding pull counts of suboptimal arms: spikes + leakage}
We now use a simple ``spikes + leakage'' decomposition for empirical means.

Fix $\delta$ and thresholds $\theta=1-2\delta$, $\theta'=1-3\delta$ as above.
For a fixed arm $i$, let
\[
\widehat{\mu}_i(n)\ :=\ \frac{1}{n}\sum_{s=1}^n X_{i,s}
\]
be the empirical mean of the first $n$ rewards of arm $i$.

\begin{lemma}[Spikes of empirical means are exponentially rare]\label{lem:posterior-spikes}
For any arm mean $\mu<\theta'$ and any $n\ge 1$,
\[
\mathbb{P}\big(\widehat{\mu}_i(n)\ge \theta' \ \big| \ \mu_i=\mu\big)\ \le\ \exp\big(-2n(\theta'-\mu)^2\big).
\]
Consequently,
\[
\sum_{n=1}^\infty \mathbb{P}\big(\widehat{\mu}_i(n)\ge \theta' \ \big| \ \mu_i=\mu\big)\ \le\ \frac{1}{2(\theta'-\mu)^2}.
\]
\end{lemma}

\begin{proof}
The first claim is Hoeffding's inequality for Bernoulli averages.
For the second, write $a=\theta'-\mu>0$ and sum the geometric bound:
\[
\sum_{n=1}^\infty e^{-2na^2}
=
\frac{e^{-2a^2}}{1-e^{-2a^2}}
=
\frac{1}{e^{2a^2}-1}
\le
\frac{1}{2a^2},
\]
using $e^x-1\ge x$ for $x\ge 0$.
\end{proof}

\paragraph{Good arms provide a uniform score floor.}
On the event $\mathcal{G}_i(\delta)$ in \eqref{eq:Gi-def}, we have $\widehat{\mu}_i(n)>\theta$ for every $n\ge 1$.
Since each arm is pulled once during initialization, it follows that for every such arm and every $t\ge m+1$,
\[
\widehat{\mu}_{i,t-1}
=
\widehat{\mu}_i\big(N_i(t-1)\big)
>
\theta.
\]
Thus, if $M(\delta)\ge r$, then throughout the softmax phase there are always at least $r$ arms whose current empirical means exceed $\theta$.
This is the key simplification in the empirical-mean proof: no maturation argument is needed.

\subsection{Pull-count bound conditional on having $r$ good arms}
We now state the pull-count lemma.

\begin{lemma}[Pull-count bound under $r$ reference arms]\label{lem:pull-count}
Fix an integer $r \in \{1, \dots, m\}$, fix $\delta$, and set thresholds $\theta=1-2\delta$, $\theta'=1-3\delta$.
For each $t\in\{m+1,\dots,T\}$, define the $\mathcal{F}_{t-1}$-measurable event
\[
\mathcal{E}_t
:=
\Big\{
\exists \mathcal{J}\subset[m]\ \text{with}\ |\mathcal{J}|=r
\ \text{such that}\ 
\widehat{\mu}_{j,t-1}\ge \theta\ \forall j\in\mathcal{J}
\Big\}.
\]
Then for any arm $i$ and any horizon $T$,
\begin{equation}\label{eq:pull-count-bound}
\mathbb{E}\!\left[N_i(T)\,\mathbf{1}\Big\{\bigcap_{t=m+1}^T \mathcal{E}_t\Big\}\right]
\ \le\
1
+
\sum_{n=1}^\infty \mathbb{P}\big(\widehat{\mu}_i(n)\ge \theta'\big)
+
\frac{1}{r}\sum_{t=m+1}^T \exp(-\eta_t\,(\theta-\theta')).
\end{equation}
\end{lemma}

\begin{proof}
Let $\tau_{i,n}$ be the time of the $n$-th pull of arm $i$, with $\tau_{i,n}=\infty$ if the $n$-th pull never occurs.
If $\tau_{i,n+1}\le T$, then at time $\tau_{i,n+1}-1$ the arm has been observed exactly $n$ times, hence
\[
\widehat{\mu}_{i,\tau_{i,n+1}-1} = \widehat{\mu}_i(n).
\]
Therefore
\[
N_i(T)\,\mathbf{1}\Big\{\bigcap_{t=m+1}^T \mathcal{E}_t\Big\}
\le
1
+
\sum_{n=1}^\infty
\mathbf{1}\Big\{\tau_{i,n+1}\le T,\ \widehat{\mu}_i(n)\ge \theta'\Big\}
+
\sum_{t=m+1}^T
\mathbf{1}\Big\{A_t=i,\ \widehat{\mu}_{i,t-1}<\theta',\ \mathcal{E}_t\Big\}.
\]
Taking expectations, the first sum is bounded by
\[
\sum_{n=1}^\infty \mathbb{P}\big(\widehat{\mu}_i(n)\ge \theta'\big).
\]

For the second sum, note that $\mathcal{E}_t\in\mathcal{F}_{t-1}$.
Hence
\[
\mathbb{E}\big[\mathbf{1}\{A_t=i,\ \widehat{\mu}_{i,t-1}<\theta',\ \mathcal{E}_t\}\big]
=
\mathbb{E}\Big[
\mathbf{1}\{\widehat{\mu}_{i,t-1}<\theta',\ \mathcal{E}_t\}\,
\mathbb{P}(A_t=i\mid \mathcal{F}_{t-1})
\Big].
\]
On $\mathcal{E}_t\cap\{\widehat{\mu}_{i,t-1}<\theta'\}$, Lemma~\ref{lem:softmax-pointwise} gives
\[
\mathbb{P}(A_t=i\mid \mathcal{F}_{t-1})
\le
\frac{1}{r}\exp\big(-\eta_t(\theta-\theta')\big).
\]
Summing over $t$ yields \eqref{eq:pull-count-bound}.
\end{proof}

\subsection{A tail-moment lemma for integrating the spike bound}
To convert \eqref{eq:pull-count-bound} into Bayes regret, we need to integrate over $\mu$ the spike penalty.
This is exactly where $\beta$-regularity yields logarithmic (for $\beta=1$) or polynomial moment control.

\begin{lemma}[Tail-moment bound, general $\beta$]\label{lem:tail-moment-general}
Let $\Gamma$ be $\beta$-regular (Definition~\ref{def:beta-regular}) and let $\mu\sim\Gamma$.
Fix $\delta\in(0,\varepsilon_0/8]$ and set $U:=1-\mu$.
Then there is a constant $C<\infty$ (depending only on $\beta, c_0, C_0, \varepsilon_0$) such that
\[
\mathbb{E}\left[\mathbf{1}\{U\ge 4\delta\}\left(1+\frac{1}{U-3\delta}\right)\right]
\ \le\
C\cdot \mathfrak{M}_\beta(\delta),
\]
where
\[
\mathfrak{M}_\beta(\delta)
:=
\begin{cases}
1+\log(1/\delta), & \beta=1,\\
1, & \beta>1,\\
\delta^{-(1-\beta)}, & \beta\in(0,1).
\end{cases}
\]
\end{lemma}

\begin{proof}
The proof is a straightforward integration-by-parts argument utilizing the $\beta$-regular bound $F(u):=\mathbb{P}(U\le u)\le C_0 u^\beta$ for $u \in (0, \varepsilon_0]$. Constants in the bounds below are allowed to depend on $(\beta, c_0, C_0, \varepsilon_0)$.

Writing
\[
\mathbb{E}\left[\frac{1}{U-3\delta}\mathbf{1}\{U\ge 4\delta\}\right]
=\int_{4\delta}^1 \frac{1}{u-3\delta}\,dF(u),
\]
integration by parts yields the boundary contribution
\[
\left[\frac{F(u)}{u-3\delta}\right]_{4\delta}^1 = \frac{F(1)}{1-3\delta} - \frac{F(4\delta)}{\delta}.
\]
Using the left-limit convention at the lower endpoint, the lower-boundary term is $-F((4\delta)^-)/\delta$, which is non-positive since the CDF $F$ is non-negative; dropping it preserves the upper bound (and automatically absorbs any atom $\Gamma$ may place at $u = 4\delta$). The boundary contribution is then at most $F(1)/(1-3\delta) = O(1)$. For the remaining integral, since $\delta \le \varepsilon_0/8$ implies $u - 3\delta \ge \varepsilon_0/2$ on $[\varepsilon_0, 1]$, the contribution from $[\varepsilon_0, 1]$ is $O(1)$ (with constants depending on $\varepsilon_0$, using $F(u) \le 1$). On $[4\delta, \varepsilon_0]$, the bound $u - 3\delta \ge u/4$ gives
\[
\int_{4\delta}^{\varepsilon_0} \frac{F(u)}{(u-3\delta)^2}\,du
\ \lesssim\
\int_{4\delta}^{\varepsilon_0} \frac{u^\beta}{u^2}\,du
=
\int_{4\delta}^{\varepsilon_0} u^{\beta-2}\,du,
\]
which scales as $O(\log(1/\delta))$ for $\beta=1$, $O(1)$ for $\beta>1$, and $\Theta(\delta^{-(1-\beta)})$ for $\beta\in(0,1)$. Adding the constant $1$ term from $\mathbb{E}[\mathbf{1}\{U \ge 4\delta\}] \le 1$ completes the proof.
\end{proof}

\subsection{Order statistics: size of the subtraction term}
The next lemma records the size of the subtraction term in \eqref{eq:proxy-ineq-bayes}. It is not needed for the upper bound in Theorem~\ref{thm:asg-beta1}, but it quantifies how much the exact identity can sharpen constants.

\begin{lemma}[Expected gap to $1$ of the best of $m$ arms]\label{lem:order-stat}
If $\Gamma$ is $\beta$-regular (Definition~\ref{def:beta-regular}) and $\mu_*=\max_{i\le m}\mu_i$, then
\[
\mathbb{E}[1-\mu_*]\ \le\ \frac{C}{m^{1/\beta}}
\]
for a constant $C$ depending only on $(\beta,c_0,C_0,\varepsilon_0)$.
In particular, for $\beta=1$, $\mathbb{E}[1-\mu_*]\le C/m$.
\end{lemma}

\begin{proof}
Let $U_i:=1-\mu_i$ and $U_*:=\min_{i\le m}U_i = 1-\mu_*$.
By $\beta$-regularity, for small $u$ we have $\mathbb{P}(U_i\le u)\ge c_0 u^\beta$, hence
\[
\mathbb{P}(U_*>u)
=
\mathbb{P}(U_1>u)^m
\le
(1-c_0 u^\beta)^m
\le
\exp(-c_0 m u^\beta)
\quad (u\le \varepsilon_0).
\]
Integrating,
\[
\mathbb{E}[U_*]
=
\int_0^\infty \mathbb{P}(U_*>u)\,du
\le
\int_0^{\varepsilon_0} e^{-c_0 m u^\beta}\,du + (1-\varepsilon_0)\mathbb{P}(U_*>\varepsilon_0)
\le
C m^{-1/\beta},
\]
using the change of variables $v=c_0 m u^\beta$.
\end{proof}

\subsection{Proof of Theorem~\ref{thm:asg-beta1}}
\begin{proof}
Let
\[
\theta := 1-2\delta,
\qquad
\theta' := 1-3\delta,
\qquad
p_\delta := \mathbb{P}\big(\mathcal{G}_i(\delta)\big),
\qquad
r := \max\!\left\{1,\left\lfloor \frac{m p_\delta}{2}\right\rfloor\right\}.
\]
Define the event
\[
\mathcal{E} := \big\{M(\delta)\ge r\big\}.
\]
By Lemma~\ref{lem:many-good-arms},
\[
\mathbb{P}(\mathcal{E}^c)\ \le\ \exp(-c\,m\,\delta).
\]

For each $t\in\{m+1,\dots,T\}$ let $\mathcal{E}_t$ be the event from Lemma~\ref{lem:pull-count}.
If $\mathcal{E}$ occurs, then there are at least $r$ arms $j$ such that $\mathcal{G}_j(\delta)$ holds.
For each such arm and each $t\ge m+1$,
\[
\widehat{\mu}_{j,t-1}
=
\widehat{\mu}_j\big(N_j(t-1)\big)
>
\theta,
\]
so $\mathcal{E}\subseteq \bigcap_{t=m+1}^T \mathcal{E}_t$.
Hence Lemma~\ref{lem:pull-count} implies that for every arm $i$,
\begin{equation}\label{eq:thm-pull-count}
\mathbb{E}\big[N_i(T)\mathbf{1}\{\mathcal{E}\}\big]
\ \le\
1 + \sum_{n=1}^\infty \mathbb{P}\big(\widehat{\mu}_i(n)\ge \theta'\big)
+ \frac{1}{r}\sum_{t=m+1}^T e^{-\eta_t\delta}.
\end{equation}

Write $U_i:=1-\mu_i$.
Decompose the surrogate regret on $\mathcal{E}$ as
\[
\widetilde{R}_T\,\mathbf{1}\{\mathcal{E}\}
=
\sum_{i=1}^m U_i N_i(T)\mathbf{1}\{\mathcal{E}\}\mathbf{1}\{U_i<4\delta\}
+
\sum_{i=1}^m U_i N_i(T)\mathbf{1}\{\mathcal{E}\}\mathbf{1}\{U_i\ge 4\delta\}.
\]

For the near-optimal arms,
\[
\sum_{i=1}^m U_i N_i(T)\mathbf{1}\{\mathcal{E}\}\mathbf{1}\{U_i<4\delta\}
\le
4\delta \sum_{i=1}^m N_i(T)
=
4\delta T,
\]
hence
\begin{equation}\label{eq:near-optimal-part}
\mathbb{E}\Big[\sum_{i=1}^m U_i N_i(T)\mathbf{1}\{\mathcal{E}\}\mathbf{1}\{U_i<4\delta\}\Big]
\le
4T\delta.
\end{equation}

For the remaining arms, we redo the pathwise decomposition from the proof of Lemma~\ref{lem:pull-count}, this time keeping the factor $U_i\mathbf{1}\{U_i\ge 4\delta\}$. On $\mathcal{E}\subseteq\bigcap_{t=m+1}^T\mathcal{E}_t$, that proof gives the pathwise inequality
\[
N_i(T)\,\mathbf{1}\{\mathcal{E}\}
\ \le\
1 + \sum_{n=1}^\infty \mathbf{1}\{\widehat{\mu}_i(n)\ge \theta'\}
+ \sum_{t=m+1}^T \mathbf{1}\{A_t=i,\ \widehat{\mu}_{i,t-1}<\theta',\ \mathcal{E}_t\}.
\]
Multiply both sides by the nonnegative quantity $U_i\mathbf{1}\{U_i\ge 4\delta\}$, sum over $i$, and take expectation. The first two pathwise terms give the spike part on the right-hand side of \eqref{eq:suboptimal-decomp} below; these will be handled by conditioning on $\mu_i$. For the leakage part (the last sum), use the pathwise bound $U_i\mathbf{1}\{U_i\ge 4\delta\}\le 1$ \emph{before} taking expectation, and then repeat the conditioning step from the proof of Lemma~\ref{lem:pull-count}: on $\mathcal{E}_t\cap\{\widehat{\mu}_{i,t-1}<\theta'\}$, Lemma~\ref{lem:softmax-pointwise} bounds $\mathbb{P}(A_t=i\mid \mathcal{F}_{t-1}) \le \tfrac{1}{r}e^{-\eta_t\delta}$. The result is
\begin{align}
\mathbb{E}\Big[\sum_{i=1}^m U_i N_i(T)\mathbf{1}\{\mathcal{E}\}\mathbf{1}\{U_i\ge 4\delta\}\Big]
&\le
\sum_{i=1}^m
\mathbb{E}\Big[
\mathbf{1}\{U_i\ge 4\delta\}\,
U_i\Big(1+\sum_{n=1}^\infty \mathbf{1}\{\widehat{\mu}_i(n)\ge \theta'\}\Big)
\Big]
+
\frac{m}{r}\sum_{t=m+1}^T e^{-\eta_t\delta}.
\label{eq:suboptimal-decomp}
\end{align}

Now condition on $\mu_i$. Since the summands in the spike series are nonnegative, Tonelli's theorem justifies interchanging the expectation and the infinite sum.
If $U_i\ge 4\delta$, then $\mu_i\le 1-4\delta<\theta'$, so Lemma~\ref{lem:posterior-spikes} gives
\[
\sum_{n=1}^\infty
\mathbb{P}\big(\widehat{\mu}_i(n)\ge \theta' \mid \mu_i\big)
\le
\frac{1}{2(\theta'-\mu_i)^2}
=
\frac{1}{2(U_i-3\delta)^2}.
\]
Therefore
\[
\mathbb{E}\Big[
\mathbf{1}\{U_i\ge 4\delta\}\,
U_i\Big(1+\sum_{n=1}^\infty \mathbf{1}\{\widehat{\mu}_i(n)\ge \theta'\}\Big)
\Big]
=
\mathbb{E}\Big[
\mathbf{1}\{U_i\ge 4\delta\}\,
U_i\Big(1+\sum_{n=1}^\infty \mathbb{P}(\widehat{\mu}_i(n)\ge \theta'\mid \mu_i)\Big)
\Big]
\]
\[
\le
\mathbb{E}\Big[
\mathbf{1}\{U_i\ge 4\delta\}\,
U_i\Big(1+\frac{1}{2(U_i-3\delta)^2}\Big)
\Big].
\]
On $\{U_i\ge 4\delta\}$ we have $U_i-3\delta\ge \delta$ and $U_i\le 4(U_i-3\delta)$, hence
\[
U_i\Big(1+\frac{1}{2(U_i-3\delta)^2}\Big)
\le
4(U_i-3\delta)+\frac{2}{U_i-3\delta}
\le
6\Big(1+\frac{1}{U_i-3\delta}\Big).
\]
Applying Lemma~\ref{lem:tail-moment-general} with $\beta=1$ yields
\begin{equation}\label{eq:tail-moment-applied}
\mathbb{E}\Big[
\mathbf{1}\{U_i\ge 4\delta\}\,
U_i\Big(1+\sum_{n=1}^\infty \mathbf{1}\{\widehat{\mu}_i(n)\ge \theta'\}\Big)
\Big]
\le
C\bigl(1+\log(1/\delta)\bigr).
\end{equation}
Summing \eqref{eq:tail-moment-applied} over $i$ and combining with \eqref{eq:suboptimal-decomp} gives
\begin{equation}\label{eq:suboptimal-final}
\mathbb{E}\Big[\sum_{i=1}^m U_i N_i(T)\mathbf{1}\{\mathcal{E}\}\mathbf{1}\{U_i\ge 4\delta\}\Big]
\le
C\,m\bigl(1+\log(1/\delta)\bigr)
+
\frac{m}{r}\sum_{t=m+1}^T e^{-\eta_t\delta}.
\end{equation}

By Lemma~\ref{lem:pdelta-lower}, $p_\delta\ge c\delta$ for all sufficiently small $\delta$.
If $m p_\delta\ge 2$, then $r=\lfloor m p_\delta/2\rfloor$ and therefore
\[
\frac{m}{r}\le \frac{4}{p_\delta}\le \frac{C}{\delta}.
\]
If $m p_\delta<2$, then $r=1$ and
\[
m < \frac{2}{p_\delta}\le \frac{C}{\delta},
\]
so again $m/r\le C/\delta$.
Combining \eqref{eq:near-optimal-part} and \eqref{eq:suboptimal-final},
\[
\mathbb{E}\big[\widetilde{R}_T\,\mathbf{1}\{\mathcal{E}\}\big]
\le
C\left[
T\delta
+
m\bigl(1+\log(1/\delta)\bigr)
+
\frac{1}{\delta}\sum_{t=1}^T e^{-\eta_t\delta}
\right],
\]
where we enlarged the sum from $t=m+1$ to $t=1$.

Finally,
\[
\mathbb{E}[\widetilde{R}_T]
=
\mathbb{E}\big[\widetilde{R}_T\,\mathbf{1}\{\mathcal{E}\}\big]
+
\mathbb{E}\big[\widetilde{R}_T\,\mathbf{1}\{\mathcal{E}^c\}\big]
\le
\mathbb{E}\big[\widetilde{R}_T\,\mathbf{1}\{\mathcal{E}\}\big]
+
T\,\mathbb{P}(\mathcal{E}^c),
\]
so
\[
\mathbb{E}[\widetilde{R}_T]
\le
C\left[
m
+
T\delta
+
m\bigl(1+\log(1/\delta)\bigr)
+
\frac{1}{\delta}\sum_{t=1}^T e^{-\eta_t\delta}
+
T e^{-c m\delta}
\right].
\]
Using \eqref{eq:proxy-ineq-bayes}, we have
\[
\mathrm{BR}_{T,m}(\text{ASG})
\le
\mathbb{E}[\widetilde{R}_T],
\]
which proves \eqref{eq:asg-main-bound}.

For the specialization \eqref{eq:delta-choice-beta1}, set $\delta = \min\{\delta_0, A\log(T\vee 2)/m\}$ with $A > 1/c$, and $\eta_t = (c_\eta/\delta)\log(t\vee 2)$ with $c_\eta > 1$. Then $\sum_{t=1}^T e^{-\eta_t\delta} = O(1)$.

If $A\log(T\vee 2)/m \le \delta_0$, then $\delta = A\log(T\vee 2)/m$ and $T e^{-c m \delta} = T^{1-cA} \le 1$ since $cA > 1$. The other terms in \eqref{eq:asg-main-bound} are $m$, $T\delta = A\log(T\vee 2)\cdot T/m$, $m(1+\log(1/\delta)) = \tilde{O}(m)$, and $(1/\delta)\sum_t e^{-\eta_t\delta} = \tilde{O}(m)$, summing to $\tilde{O}(m + T/m)$.

If $A\log(T\vee 2)/m > \delta_0$ (i.e., $m < A\log(T\vee 2)/\delta_0$), the cap binds and $T/m > \delta_0 T/(A\log(T\vee 2)) = \tilde{\Omega}(T)$, so $\tilde{O}(m+T/m)$ already absorbs the trivial bound $\mathrm{BR}_{T,m}(\text{ASG}) \le T$.

In both cases, $\mathrm{BR}_{T,m}(\text{ASG}) = \tilde{O}(m + T/m)$.
\end{proof}

\section{Toward a Guarantee for Prior-ASG}
\label{app:prior-theory}

This appendix outlines how the proof of Theorem~\ref{thm:asg-beta1} could extend to Prior-ASG.

\paragraph{Setting.}
Each arm $i$ carries a Beta prior with mean $\nu_i$ and weight $\lambda_i \le \Lambda$, both known to the algorithm, with $(\mu_i, \nu_i, \lambda_i)$ i.i.d.\ across arms. Prior-ASG samples $A_t \sim \mathrm{softmax}(\eta_t\, \bar\mu_{i,t-1})$ from the first round, where
\[
\bar\mu_{i,t} \;=\; \frac{\lambda_i \nu_i + S_i(t)}{\lambda_i + N_i(t)}
\]
is the posterior mean, defined for every arm from round one. Two conditions on the prior replace $1$-regularity. The first is \emph{coverage}: there is a constant $c_1 > 0$ such that, for all small $\delta$,
\[
\mathbb{P}\big(\mu_i \ge 1 - \delta \ \text{and}\ \nu_i \ge 1 - 2\delta\big) \;\ge\; c_1\,\delta,
\]
so a fraction of arms proportional to $\delta$ is both genuinely near-optimal and scored as such by the prior. The second is \emph{calibration}: $|\nu_i - \mu_i| \le \varepsilon$ for a misspecification level $\varepsilon \ge 0$. Coverage plays the role that $1$-regularity plays in Theorem~\ref{thm:asg-beta1}; indeed, if the true means are $1$-regular, coverage holds at every scale $\delta \ge 2\varepsilon$, since every arm with $\mu_i \ge 1 - \delta$ then has $\nu_i \ge 1 - 3\delta/2 \ge 1 - 2\delta$. The argument below only uses scales $\delta \ge 2\varepsilon$.

\paragraph{Dropping the initialization.}
The posterior mean is a weighted average of the prior mean and the empirical mean,
\[
\bar\mu_i(n) \;=\; \frac{\lambda_i}{\lambda_i + n}\,\nu_i \;+\; \frac{n}{\lambda_i + n}\,\widehat\mu_i(n),
\]
where $n$ is the arm's pull count. Consider an arm with $\mu_i \ge 1 - \delta$ and $\nu_i \ge 1 - 2\delta$. On the crossing event of Lemma~\ref{lem:bern-cross}, its empirical mean stays above $1 - 2\delta$ at every pull count $n \ge 1$, and the display then keeps its score at or above $1 - 2\delta$ for every $n \ge 0$, including $n = 0$. In the proof of Theorem~\ref{thm:asg-beta1}, the always-good arms of Lemma~\ref{lem:many-good-arms} acquire their score floor only after the forced initialization pull; here they hold it from the first round without any pull. Coverage then supplies $r \asymp m\delta$ such arms outside an event of probability $e^{-c m \delta}$, exactly as the binomial count of Lemma~\ref{lem:many-good-arms} does in the current proof.

A good part of the argument behind Theorem~\ref{thm:asg-beta1} can now be reused. With $r$ arms holding scores at or above $1 - 2\delta$, Lemma~\ref{lem:softmax-pointwise} bounds the probability of selecting any given arm whose score sits below $1 - 3\delta$ by $r^{-1} e^{-\eta_t \delta}$. Under the schedule~\eqref{eq:schedule} the sum over $t$ is $O(1/r)$ for each such arm, hence $O(m/r) = O(1/\delta)$ in total across arms, which is the same leakage total the current proof carries into the theorem's bound. The bad-event term $T e^{-c m \delta}$ also carries over, with coverage in place of the $1$-regular tail bound of Lemma~\ref{lem:pdelta-lower}.

The remaining part of the argument requires adjustment. The $m$ and $m(1 + \log(1/\delta))$ terms in Theorem~\ref{thm:asg-beta1} come from the pull-count argument (Lemma~\ref{lem:pull-count}), which charges every arm its initialization pull plus the possible excursions of its empirical mean above the threshold. Without initialization, a different accounting is available. Call an arm bad if $\mu_i \le 1 - 5\delta$; arms in between cost at most $5\delta$ per pull, so their total contribution is at most $5T\delta$, absorbed by the $T\delta$ term as in the proof of Theorem~\ref{thm:asg-beta1}. Since $\delta \ge 2\varepsilon$, calibration places a bad arm's prior score below the bar $1 - 3\delta$, so the policy can select it for the first time only through a leakage step, and the expected number of such entries, totaled across bad arms, is the $O(1/\delta)$ above. After an entry, write
\[
W_n \;:=\; (\lambda_i + n)\big(\bar\mu_i(n) - (1 - 3\delta)\big),
\]
which changes by $X_{i,n} - (1 - 3\delta)$ at the $n$-th pull of the arm and therefore drifts downward by at least $2\delta$ per pull in expectation. The time the arm's score spends above the bar is not constant (for an arm just below the cutoff it is of order $1/\delta$), but each such round costs $1 - \mu_i \ge 5\delta$ in regret-to-one, and a stopped-random-walk argument, or a multiplicative Chernoff bound on the arm's success count, should bound the product, the expected regret per entry, by a constant uniformly over $\mu_i \le 1 - 5\delta$. Bad arms would then cost $O(1)$ regret per entry, $O(1/\delta)$ in total, rather than a charge per arm for existing, and the $m$-scale terms would disappear from the bound.

Combining these steps, we expect a result of the following type to be provable: under an arm-specific prior satisfying coverage and calibration with misspecification level $\varepsilon$, Prior-ASG with the schedule~\eqref{eq:schedule} at a suitable threshold $\delta$ achieves Bayes regret $\tilde{O}\big(T\varepsilon + \sqrt{T} + T/m\big)$, with no requirement that $T \ge m$. In particular, when the prior is accurate to order $T^{-1/2}$ and $m \ge \sqrt{T}$, the Bayes regret stays $\tilde{O}(\sqrt{T})$ no matter how large $m$ grows. We leave the precise statement and its proof for future work.

Arm-specific prior information is what makes this argument possible. Under a common prior every unpulled arm has the same score, so the prior cannot single out which arms are worth starting from; the $\nu = 1/2$ ablation in Appendix~\ref{app:nu-ablation} shows this failure empirically. The common optimistic prior of Section~\ref{sec:sim-prior} works through a different mechanism, a rotation through fresh arms similar to the sequential greedy procedure of \citet{bayati2020greedy}, and lies outside the argument sketched here.

\section{Extended Related Work}
\label{app:related}

This appendix expands on the condensed survey in Section~\ref{sec:related}, providing additional context and citations.

\paragraph{Stochastic bandits: optimism and posterior sampling.}
The modern theory of stochastic multi-armed bandits characterizes the exploration--exploitation trade-off through regret guarantees and instance-dependent lower bounds \cite{lai1985adaptive,auer2002finite,audibert2007tuning,bubeck2012regret,lattimore2020bandit}.
A canonical Bayesian approach is posterior sampling (Thompson sampling), originally proposed in \cite{thompson1933likelihood} and analyzed in modern finite-time settings by, e.g., \cite{agrawal2012analysis,russo2016info}.
Related Bayesian decision rules include Bayesian index policies \cite{kaufmann2018bayesian} and the classical Gittins index \cite{gittins1979bandit}.
Our setting adopts the standard Beta--Bernoulli conjugate model, which enables a clean comparison between probability matching (posterior sampling) and softmax/Boltzmann action selection.

\paragraph{Boltzmann (softmax) exploration and its limitations.}
Boltzmann exploration (a.k.a.\ softmax or Gibbs action selection) is widely used in reinforcement learning and bandits as a simple randomized alternative to greedy choice.
Despite its popularity, \citet{cesabianchi2017boltzmann} show that, for stochastic $K$-armed bandits, \emph{any monotone} temperature (learning-rate) schedule can be forced into suboptimal behavior: either it explores too long or it commits too early.
They propose remedies including (i) tuned non-monotone schedules that require knowledge of the horizon and gaps and (ii) per-arm learning rates that explicitly track estimation uncertainty \cite{cesabianchi2017boltzmann}.
This negative result motivates our focus on regimes where \emph{uncertainty-aware} schedules may be unnecessary due to structural properties of the arm distribution.

\paragraph{Many-armed and infinite-armed bandits.}
A long line of work studies bandits with infinitely many arms and tail-based performance criteria, emphasizing how the distribution of arm qualities shapes achievable regret.
Classical and modern examples include infinitely-many-armed formulations \cite{berry1997infinitely,wang2009algorithms,bonald2013twotarget,carpentier2015simple,chaudhuri2018quantile}.
In the Bayesian many-armed regime, \citet{bayati2020greedy} formalize the idea that when the prior places substantial mass on near-optimal arms (via upper-tail regularity conditions), greedy-style policies can enjoy \emph{free exploration}: discarding a poorly performing arm is likely to leave other near-optimal arms available.
They show that subsampled greedy can achieve Bayesian regret scaling of order $\tilde{O}(\max\{m,T/m\})$ (up to prior-dependent exponents and logarithms), implying near-optimal performance with $m\asymp \sqrt{T}$ arms \cite{bayati2020greedy}.
Our analysis builds on this viewpoint, but focuses on softmax/Boltzmann policies that randomize \emph{without} explicit epistemic-uncertainty bonuses.

\paragraph{Free exploration beyond non-contextual bandits.}
Complementary ``free exploration'' phenomena have been identified in contextual bandits, where exploration can be induced by natural diversity in contexts and data \cite{bastani2020mostly,kannan2018smoothed,raghavan2018externalities,hao2020adaptive,abbasi2011oful}.
These works highlight that the need for explicit exploration is sensitive to structural assumptions (e.g., covariate diversity, smoothed analysis, or rich action sets).
Our work studies an orthogonal mechanism: even \emph{without} context diversity, the presence of many near-optimal arms (captured by prior tail regularity) can make epistemic-uncertainty-agnostic softmax policies achieve near-greedy Bayes regret.

\paragraph{Satisficing and multiple near-optimal actions.}
When near-optimal actions are plentiful, identifying the unique optimal action may be information-inefficient.
This perspective is formalized in satisficing and information-theoretic approaches to bandits \cite{russo2014learning,russo2014posterior,russo2018satisficing}.
Conceptually, our ``many near-optimal arms'' regime is aligned with satisficing: regret can remain small even if the learner settles on an $\varepsilon$-optimal arm, provided such arms are sufficiently common under the prior.

\paragraph{RLVR and group-based policy optimization for LLM reasoning.}
The empirical motivation for this paper comes from the rapid adoption of reinforcement learning with verifiable (often binary) rewards for post-training large language models (LLMs) on reasoning-centric tasks.
GRPO was introduced in DeepSeekMath as a memory-efficient alternative to PPO-style actor--critic training for verifiable rewards \cite{shao2024deepseekmath,schulman2017ppo}; it has since become a widely used baseline for RLVR training at scale \cite{deepseekai2025r1,yu2025dapo}.
Several recent works analyze GRPO/RLVR training dynamics and relate them to KL-regularized reweighting.
For example, \citet{mroueh2025grpo} derive explicit optimal-policy forms for variants of GRPO under binary rewards and show how success probability can be amplified through iterative updates.
At the same time, there is ongoing debate about whether RLVR expands a model's reasoning \emph{support} (high-$k$ coverage) or mainly reweights probability mass within the base model's existing support.
\citet{yue2025rlvr} report that RLVR often improves small-$k$ performance while failing to improve---and sometimes degrading---large-$k$ pass@k, suggesting a ``bounded-by-base'' effect; subsequent work revisits this question and argues that RLVR can extend reasoning boundaries under specific protocols and metrics \cite{wen2025implicit,liu2025prorl}.
Additional analyses emphasize training instabilities, entropy collapse, and variance reduction techniques in RLVR pipelines \cite{cui2025entropy,zeng2025shrinkage,liu2025r1zero,yu2025dapo}.

\paragraph{Pass@$k$ as a tail metric and as an objective.}
Pass@k was popularized as a functional-correctness evaluation metric for code generation and is an order-statistic probe of a model's upper tail over solutions \cite{chen2021evaluating}.
Recent RLVR-specific work argues that pass@k is best interpreted as a diagnostic of exploration/coverage rather than a direct optimization target \cite{yu2025passk}, and proposes alternative objectives to mitigate probability concentration and improve pass@k at larger $k$ \cite{peng2025simko}.
Our bandit model adopts an explicit tail-regularity assumption on arm quality; this provides a stylized bridge to pass@k phenomena by translating ``good pass@k'' into ``many near-optimal arms with non-negligible mass.''

\paragraph{Connections to KL-regularized policy learning.}
A recurring theme in RLHF/RLVR is that KL regularization (to a reference or previous policy) acts as a trust-region or mirror-descent constraint, inducing stochastic policies that resemble Gibbs distributions \cite{schulman2015trpo,schulman2017ppo,ouyang2022training}.
Very recent theory work in contextual bandits and RLHF argues that KL regularization alone can induce sufficient exploration under suitable coverage assumptions \cite{zhao2025sharp}.
Our contributions are complementary: we identify a distinct mechanism---prior tail mass / abundance of near-optimal arms---under which uncertainty-agnostic annealed softmax can achieve strong Bayesian regret, offering an alternative explanation for why soft reweighting can be effective even without explicit epistemic exploration.

\paragraph{Surveys and resources.}
For broader perspective on RL for large reasoning models, including RLVR training recipes, datasets, and open problems, see the recent survey \cite{zhang2025survey}.
Reasoning Gym provides a large suite of procedurally generated, verifiable environments intended for RLVR research \cite{stojanovski2025reasoninggym}.

\section{UCB1 in the Many-Armed Regime}
\label{app:ucb1-supplement}

Figure~\ref{fig:regret-ucb1} adds UCB1 to the uninformative-regime comparison at both horizons. UCB1 pulls each arm once, then pulls the arm maximizing $\widehat{\mu}_{i,t-1} + \sqrt{2\log\!\big(1 + t\log(t^2)\big)/N_i(t-1)}$. The algorithm ignores the prior, so one run per $(T, m)$ pair covers all three regimes; we plot it against the uninformative panels, with the $y$-axis on a log scale so the other curves stay readable.

The per-arm exploration bonus forces UCB1 to keep returning to every arm, so its regret keeps growing across the entire grid at both horizons.

\begin{figure}[t]
    \centering
    \includegraphics[width=\textwidth]{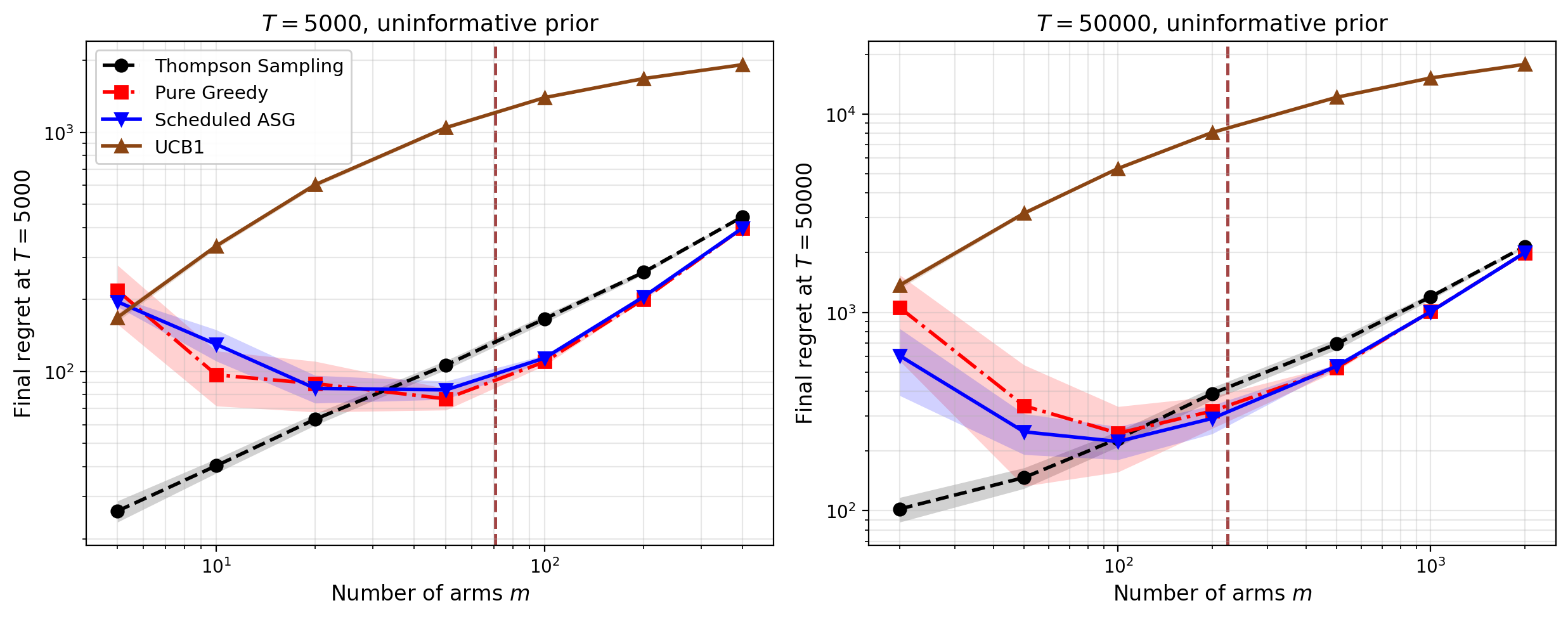}
    \caption{UCB1 against Thompson Sampling, Pure Greedy, and Scheduled ASG in the uninformative regime, at $T = 5000$ (left, $300$ trials for the baselines and $100$ for UCB1) and $T = 50000$ (right, $60$ trials). The $y$-axis is on a log scale; bands are $95\%$ confidence intervals; the vertical dashed line marks $m^\star = \sqrt{T}$. UCB1 regret keeps growing with $m$ at both horizons, and by the right end of each grid it sits far above every other algorithm.}
    \label{fig:regret-ucb1}
\end{figure}

\section{Sensitivity of Prior-Greedy and Prior-ASG to the Optimism Level}
\label{app:nu-ablation}

In the uninformative regime, Section~\ref{sec:sim-prior} runs Prior-Greedy and Prior-ASG with a common optimistic prior mean $\nu = 0.9$. Figure~\ref{fig:regret-nu-ablation} varies $\nu$ over $\{0.5, 0.7, 0.9, 0.99\}$ at $T = 5000$ with $\lambda = 1$ and $300$ trials per point, with Pure Greedy as the reference.

At $\nu = 1/2$ the commitment problem described in Section~\ref{sec:sim-prior} is visible across the whole grid for Prior-Greedy, and from $m = 20$ on for Prior-ASG, whose softmax spreading masks it at the smallest arm counts. Raising $\nu$ to $0.7$ roughly halves the damage but still loses to Pure Greedy over most of the grid. At $\nu = 0.9$ the failure is gone. The value $\nu = 0.9$ balances two demands: high enough that one failure drops an arm below the pool of unpulled arms, and low enough that one success lifts an arm clearly above that pool.

The two variants are different at $\nu = 0.99$. Prior-Greedy stays close to its $\nu = 0.9$ behavior, since the $\argmax$ rule needs only the ordering of scores. For Prior-ASG the regret climbs back at $m = 400$, close to the baseline value of Scheduled ASG. The explanation is that, after one success, an arm scores $0.995$, barely above the $0.99$ of every unpulled arm, so at the resolution of the softmax the winners are nearly indistinguishable from the unpulled pool and the policy spreads its pulls over all of them, which reproduces the initialization cost the variant was meant to avoid.

\begin{figure}[t]
    \centering
    \includegraphics[width=\textwidth]{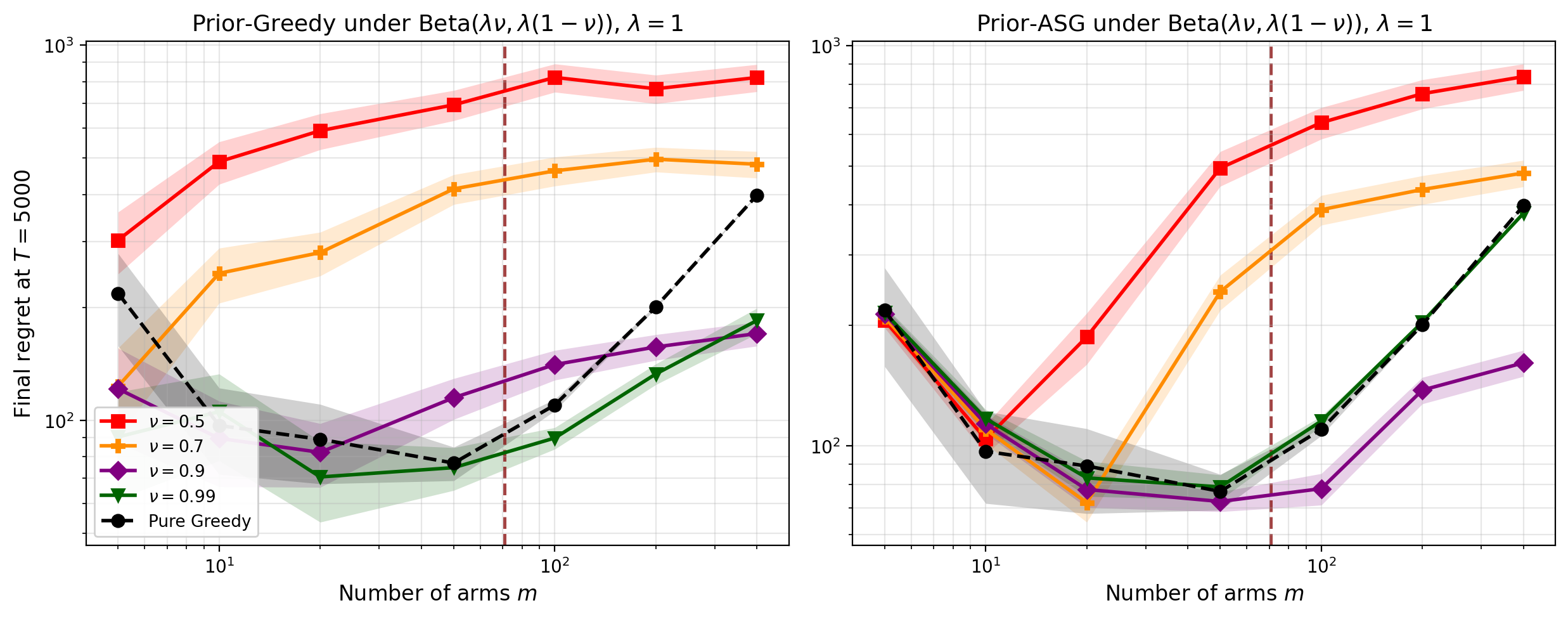}
    \caption{Sensitivity of Prior-Greedy (left) and Prior-ASG (right) to the common optimistic prior mean $\nu$, in the uninformative regime at $T = 5000$ with $\lambda = 1$ and $300$ trials per point. The $y$-axis is on a log scale; bands are $95\%$ confidence intervals; the vertical dashed line marks $m^\star = \sqrt{T} \approx 71$; the black dashed curve is Pure Greedy. Regret improves as $\nu$ rises from $1/2$ to $0.9$ (from $m = 20$ on for Prior-ASG, whose curves are indistinguishable at the smallest $m$); at $\nu = 0.99$ Prior-ASG falls back to the baseline level while Prior-Greedy does not.}
    \label{fig:regret-nu-ablation}
\end{figure}

\end{document}